\documentclass[journal,comsoc]{IEEEtran}
\usepackage[T1]{fontenc}

\ifCLASSINFOpdf
\else
\fi
\usepackage[cmintegrals]{newtxmath}
\usepackage{url}

\usepackage{cite}
\usepackage{amsmath}

\usepackage{amsthm}
\usepackage{mathtools}
\usepackage{algorithmic}
\usepackage[ruled,linesnumbered]{algorithm2e}
\usepackage{graphicx}
\usepackage{textcomp}
\usepackage{xcolor}
\usepackage{bm}
\usepackage{enumitem}
\usepackage{subfigure}
\usepackage{url}
\usepackage{hyperref}

\newtheorem{assumption}{Assumption}

\newtheorem{lemma}{Lemma}
\newtheorem{theorem}{Theorem}
\newtheorem{corollary}{Corollary}

\newcommand{\etal}{\textit{et al}.}
\newcommand{\ie}{\textit{i}.\textit{e}.}
\newcommand{\eg}{\textit{e}.\textit{g}.}
\newcommand{\bfw}{{\theta}}
\newcommand{\bfx}{\mathrm{x}}
\newcommand{\bfy}{\mathrm{y}}
\newcommand{\mli}[1]{\mathit{#1}}
\newcommand{\algname}[1]{{\textrm{#1}}}

\allowdisplaybreaks[4]

\newboolean{showcomments}
\setboolean{showcomments}{true}
\newcommand{\comment}[1]{\ifthenelse{\boolean{showcomments}}
	{\textcolor{red}{(Comment: #1)}}
	{}
}
\newcommand{\red}[1]{\ifthenelse{\boolean{showcomments}}
	{\textcolor{red}{#1}}
	{}
}

\begin{document}
%
\title{Efficient Federated Meta-Learning over Multi-Access Wireless Networks}
%
%
%

\author{Sheng~Yue, \IEEEmembership{Student Member,~IEEE,}
        Ju~Ren, \IEEEmembership{Member,~IEEE,}
        Jiang~Xin, \IEEEmembership{Student Member,~IEEE,} 
        Deyu~Zhang, \IEEEmembership{Member,~IEEE,}
        Yaoxue~Zhang, \IEEEmembership{Senior Member,~IEEE,}
        and~Weihua~Zhuang, \IEEEmembership{Fellow,~IEEE}
\thanks{Sheng Yue, Jiang Xin, Deyu Zhang are with the School of Computer Science and Engineering, Central South University, Changsha, 410083 China. Emails: \{sheng.yue, xinjiang, zdy876\}@csu.edu.cn.}
\thanks{Ju Ren and Yaoxue Zhang are with the Department of Computer Science and Technology, Tsinghua University, Beijing, 100084 China. Emails: \{renju,zhangyx\}@tsinghua.edu.cn.}
\thanks{Weihua Zhuang is with the Department of Electrical and Computer Engineering, University of Waterloo, Waterloo, ON, Canada. Email: wzhuang@uwaterloo.ca.}
}

\maketitle

\begin{abstract}
Federated meta-learning (FML) has emerged as a promising paradigm to cope with the data limitation and heterogeneity challenges in today's edge learning arena. However, its performance is often limited by slow convergence and corresponding low communication efficiency. In addition, since the available radio spectrum and IoT devices' energy capacity are usually insufficient, it is crucial to control the resource allocation and energy consumption when deploying FML in practical wireless networks. To overcome the challenges, in this paper, we rigorously analyze the contribution of each device to the global loss reduction in each round and develop an FML algorithm (called NUFM) with a non-uniform device selection scheme to accelerate the convergence. After that, we formulate a resource allocation problem integrating NUFM in multi-access wireless systems to jointly improve the convergence rate and minimize the wall-clock time along with energy cost. By deconstructing the original problem step by step, we devise a joint device selection and resource allocation strategy to solve the problem with theoretical guarantees. Further, we show that the computational complexity of NUFM can be reduced from $O(d^2)$ to $O(d)$ (with the model dimension $d$) via combining two first-order approximation techniques. Extensive simulation results demonstrate the effectiveness and superiority of the proposed methods in comparison with existing baselines.

\end{abstract}

\IEEEpeerreviewmaketitle

\section{Introduction}
\IEEEPARstart{T}{he} integration of Artificial Intelligence (AI) and Internet-of-Things (IoT) has led to a proliferation of studies on edge intelligence, aiming at pushing AI frontiers to the wireless network edge proximal to IoT devices and data sources \cite{park2019wireless}. It is expected that edge intelligence will reduce time-to-action latency down to milliseconds for IoT applications while minimizing network bandwidth and offering security guarantees \cite{plastiras2018edge}. However, a general consensus is that a single IoT device can hardly realize edge intelligence due to its limited computational and storage capabilities. Accordingly, it is natural to rely on collaboration in edge learning, whereby IoT devices work together to accomplish computation-intensive tasks \cite{zhang2019openei}.

Building on a synergy of federated learning \cite{mcmahan2017communication} and meta-learning \cite{finn2017model}, federated meta-learning (FML) has been proposed under a common theme of fostering edge-edge collaboration \cite{chen2018federated,jiang2019improving,lin2020collaborative}. In FML, IoT devices join forces to learn an \emph{initial shared model} under the orchestration of a central server such that current or new devices can quickly adapt the learned model to their local datasets via one or a few gradient descent steps. Notably, FML can keep all the benefits of the federated learning paradigm (such as simplicity, data security, and flexibility), while giving a more personalized model for each device to capture the differences among tasks \cite{fallah2020personalized}. Therefore, FML has emerged as a promising approach to tackle the heterogeneity challenges in federated learning and to facilitate efficient edge learning \cite{kairouz2019advances}. 

Despite its promising benefits, FML comes with new challenges. On one hand, the number of participated devices can be enormous. The uniform device selection at random, often done in the existing methods, leads to a low convergence speed \cite{yue2021inexact,fallah2020personalized}. Although recent studies \cite{nishio2019client,nguyen2020fast,chen2020joint} have characterized the convergence of federated learning and proposed non-uniform device selection mechanisms, they cannot be directly applied to FML problems due to the bias and high-order information in the stochastic gradients. On the other hand, the performance of FML in a wireless environment is highly related to its wall-clock time, including the computation time (determined by local data sizes and devices' CPU types) and communication time (depending on channel gains, interference, and transmission power) \cite{dinh2020federated,liu2021efficient}. If not properly controlled, a large wall-clock time can cause unexpected training delay and communication inefficiency. In addition, the DNN model training, which involves a large number of data samples and epochs, usually induces high computational cost, especially for sophisticated model structures consisting of millions of parameters \cite{dinh2020federated}. Due to the limited power capacity of IoT devices, energy consumption should also be properly managed to ensure system sustainability and stability \cite{chen2020joint}. In a nutshell, for the purpose of efficiently deploying FML in today's wireless systems, the strategy of device selection and resource allocation must be carefully crafted not only to accelerate the learning process, but also to control the wall-clock time of training and energy cost in edge devices. Unfortunately, despite their importance, there are limited studies on these aspects in the current literature.

In this paper, we tackle the above-mentioned challenges in two steps: 1) We develop an algorithm (called \algname{NUFM}) with a non-uniform device selection scheme to improve the convergence rate of vanilla FML algorithm; 2) based on \algname{NUFM}, we propose a resource allocation strategy (called \algname{URAL}) that jointly optimizes the convergence speed, wall-clock time, and energy consumption in the context of multi-access wireless systems. More specifically, first we rigorously quantify the contribution of each device to the convergence of FML via deriving a tight lower bound on the reduction of one-round global loss. Based on the quantitative results, we present the non-uniform device selection scheme that maximizes the loss reduction per round, followed by the NUFM algorithm. Then, we formulate a resource allocation problem for NUFM over wireless networks, capturing the trade-offs among convergence, wall-clock time, and energy cost. To solve this problem, we exploit its special structure and decompose it into two sub-problems. The first one is to minimize the computation time via controlling devices' CPU-cycle frequencies, which is solved optimally based on the analysis of the effect of device heterogeneity on the objective. The second sub-problem aims at optimizing the resource block allocation and transmission power management. It is a non-convex mixed-integer non-linear programming (MINLP) problem, and to derive a closed-form solution is a non-trivial task. Thus, after deconstructing the problem step by step, we devise an iterative method to solve it and provide a convergence guarantee. 

In summary, our main contributions are three-fold.
\begin{itemize}
	\item We provide a theoretical characterization of contribution of an individual device to the convergence of FML in each round, via establishing a tight lower bound on the one-round reduction of expected global loss. Using this quantitative result, we develop NUFM, a fast-convergent FML algorithm with non-uniform device selection;
	\item To embed NUFM in the context of multi-access wireless systems, we formulate a resource allocation problem, capturing the trade-offs among the convergence, wall-clock time, and energy consumption. By decomposing the original problem into two sub-problems and deconstructing sub-problems step by step, we propose a joint device selection and resource allocation algorithm (namely URAL) to solve the problem effectively with theoretical performance guarantees;
	\item To reduce the computational complexity, we further integrate our proposed algorithms with two first-order approximation techniques in \cite{fallah2020personalized}, by which the complexity of a one-step update in NUFM can be reduced from $O(d^2)$ to $O(d)$. We also show that our theoretical results hold in these cases;
	\item We provide extensive simulation results on challenging real-world benchmarks (\ie, Fashion-MNIST, CIFAR-10, CIFAR-100, and ImageNet) to demonstrate the efficacy of our methods.
\end{itemize} 

The remainder of this paper is organized as follows. Section~\ref{sec:related_work} briefly reviews the related work. Section~\ref{sec:preliminaries}
introduces the FML problem and standard algorithm. We present the non-uniform device selection scheme in Section~\ref{sec:nufm} and adapt the scheme for wireless networks in Section~\ref{sec:wireless_nufm}. Finally, Section~\ref{sec:extension} presents the extension to first-order approximation techniques, followed by the simulation results in Section~\ref{sec:simulation} and a conclusion drawn in Section~\ref{sec:conclusion}.

\section{Related Work}
\label{sec:related_work}
Federated learning (FL) \cite{mcmahan2017communication} has been proposed as a promising technique to facilitate edge-edge collaborative learning \cite{fadlullah2020hcp}. However, due to the heterogeneity in devices, models, and data distributions, a shared global model often fails to capture the individual information of each device, leading to performance degradation in inference or classification \cite{wu2020personalized,kairouz2019advances,yang2019federated}.

Very recently, based on the advances in meta-learning \cite{finn2017model}, federated meta-learning (FML) has garnered much attention, which aims to learn a personalized model for each device to cope with the heterogeneity challenges \cite{chen2018federated,jiang2019improving,lin2020collaborative,fallah2020personalized,yue2021inexact}. Chen \etal \cite{chen2018federated} first introduce an FML method called FedMeta, integrating the model-agnostic meta-learning (MAML) algorithm \cite{finn2017model} into the federated learning framework. They show that FML can significantly improve the performance of FedAvg \cite{mcmahan2017communication}. Jiang \etal \cite{jiang2019improving} analyze the connection between FedAvg and MAML, and empirically demonstrate that FML enables better and more stable personalized performance. From a theoretical perspective, Lin \etal \cite{lin2020collaborative} analyze the convergence properties and computational complexity of FML with strongly convex loss functions and exact gradients. Fallah \etal \cite{fallah2020personalized} further provide the convergence guarantees in non-convex cases with stochastic gradients. Different from the above gradient descent--based approaches, another recent work \cite{yue2021inexact} develops an ADMM-based FML method and gives its convergence guarantee under non-convex cases. However, due to selecting devices uniformly at random, the existing FML algorithms often suffer from slow convergence and low communication efficiency \cite{fallah2020personalized}. Further, deploying FML in practical wireless systems calls for effective resource allocation strategies \cite{dinh2020federated,yue2021todg}, which is beyond the scope of existing FML literature.

There exists a significant body of works placing interests in convergence improvement and resource allocation for FL \cite{ren2020scheduling,zhu2019broadband,zeng2020energy,shi2020joint,nguyen2020fast,chen2020convergence,ren2020accelerating,yu2019linear,dinh2020federated,xu2020client,wang2019adaptive,chen2020joint,karimireddy2020scaffold,luo2021cost,luping2019cmfl,li2019differentially,sim2019personalization,shi2020device,yang2020energy}. Regarding the convergence improvement, Nguyen \etal \cite{nguyen2020fast} propose a fast-convergent FL algorithm, called FOLB, which achieves a near-optimal lower bound for the overall loss decrease in each round. Note that, while the idea of NUFM is similar to \cite{nguyen2020fast}, the lower bound for FML is derived from a completely different technical path due to the inherent complexity in the local update. To minimize the convergence time, a probabilistic device selection scheme for FL is designed in \cite{chen2020convergence}, which assigns high probabilities to the devices with large effects on the global model. Ren \etal \cite{ren2020accelerating} investigate a batchsize selection strategy for accelerating the FL training process. Karimireddy \etal \cite{karimireddy2020scaffold} employ the variance reduction technique to develop a new FL algorithm. Based on momentum methods, Yu \etal \cite{yu2019linear} give an FL algorithm with linear speedup property. Regarding the resource allocation in FL, Dinh \etal \cite{dinh2020federated} embed FL in wireless networks, considering the trade-offs between training time and energy consumption, under the assumption that all devices participate in the whole training process. From a long-term perspective, Xu \etal \cite{xu2020client} empirically investigate the device selection scheme jointly with bandwidth allocation for FL, using the Lyapunov optimization method. Chen \etal \cite{chen2020joint} investigate a device selection problem with ``hard'' resource constraints to enable the implementation of FL over wireless networks. Wang \etal \cite{wang2019adaptive} propose a control algorithm to determine the best trade-off between the local update and global aggregation under a resource budget. 

Although extensive research has been carried out on FL, researchers have not treated FML in much detail. In particular, the existing FL acceleration techniques cannot be directly applied to FML due to the high-order information and biased stochastic gradients in the local update phases\footnote{Different from FML, FL is a first-order method with unbiased stochastic gradients.} (see Lemma \ref{lem:variance_Fi} in Section~\ref{sec:nufm}). At the same time, device selection and resource allocation require crafting jointly \cite{chen2018federated}, rather than simply plugging the existing strategies in.

\section{Preliminaries and Assumptions}
\label{sec:preliminaries}
In this section, we introduce federated meta-learning, including the learning problem, standard algorithm, and assumptions for theoretical analysis.

\begin{table*}[ht]
	\caption{Key Notations}
	\label{table:notations}
	\centering
	\renewcommand\arraystretch{1.25}
	\resizebox{\textwidth}{!}{
	\begin{tabular}{l|l||l|l}
		\hline
		\textbf{Notation} & \textbf{Definition} & \textbf{Notation} & \textbf{Definition}\\ \hline
		$\theta$ & Model parameter & $F_i(\cdot)$ & Meta--loss function\\
		$\alpha$, $\beta$ & Stepsize and meta--learning rate & $L_i$, $\rho_i$ & Lipschitz continuous parameters\\
		$i$, $j$ & Indexes of user devices and data samples & $\sigma_{G}$, $\sigma_{H}$ & Upper bounds of variances\\
		$k$ & Index of training rounds& $\gamma_{G}$, $\gamma_{H}$ & Similarity parameters\\
		$t$ & Index of local update steps & $u_i$ & Contribution to global loss reduction of device $i$\\
		$\tau$ & Number of local update steps & $v_i$ & CPU-cycle frequency of device $i$\\
		$P_i$ & Underlying data distribution of device $i$ & $p_i$ & Transmission power of device $i$\\
		$\mathcal{D}_i$, $D_i$ & Dataset and its size of device $i$       &$z_{i,m}$ & Binary variable indicating if device $i$ accesses RB $m$\\
		$\mathcal{N}$ & Set of devices & $\eta_1$, $\eta_2$ & Weight parameters\\
		$n$ & Number of devices & $c_i$ & CPU cycles for computing one sample by device $i$\\
		$\mathcal{N}_k$ & Set of participating devices in round $k$ & $h_i$, $I_m$ & Channel gain of device $i$ and inference in RB $m$\\
		$n_k$ & Number of participating devices in round $k$ & $B$, $N_0$ & Bandwidth of each RB and noise power spectral density\\
		$\mathrm{x}$, $\mathrm{y}$ & Input and corresponding label & $\mathcal{M}$, $M$ & Set and number of RBs\\
		$l_i(\theta;\mathrm{x},\mathrm{y})$ & Loss of model $\theta$ on sample $(\mathrm{x},\mathrm{y})$ & $E$ & Total energy consumption\\
		$f_i(\cdot)$ & Expected loss function & $T$ & Total latency \\
		\hline
	\end{tabular}
	}	
\end{table*}

\subsection{Federated Meta-Learning Problem}
\label{subsec:learning_prob}
We consider a set $\mathcal{N}$ of user devices that are all connected to a server. Each device $i\in\mathcal{N}$ has a labeled dataset $\mathcal{D}_i=\{\bfx^j_i,\bfy^j_i\}^{D_i}_{j=1}$ that can be accessed only by itself. Here, the tuple $(\bfx^j_i,\bfy^j_i)\in\mathcal{X}\times\mathcal{Y}$ is a data sample with input $\bfx^j_i$ and label $\bfy^j_i$, and follows an unknown underlying distribution $P_i$. Define $\theta$ as the model parameter, such as the weights of a Deep Neural Network (DNN) model. For device $i$, the loss function of a model parameter $\bfw\in\mathbb{R}^d$ is defined as $\ell_i(\bfw;\bfx,\bfy)$, which measures the error of model $\bfw$ in predicting the true label $\bfy$ given input $\bfx$. 

Federated meta-learning (FML) looks for a good model initialization (also called meta-model) such that the well-performed models of different devices can be quickly obtained via one or a few gradient descent steps. More specially, FML aims to solve the following problem
\begin{align}
	\label{prob:fed_meta}
	\mathop{\min}_{\bfw\in\mathbb{R}^d} F(\bfw)\coloneqq \frac{1}{n}\sum_{i\in\mathcal{N}} f_i(\bfw - \alpha\nabla f_i(\bfw))
\end{align}
where $f_i$ represents the expected loss function over the data distribution of device $i$, \ie, $f_i(\bfw)\coloneqq\mathbb{E}_{(\bfx,\bfy)\sim P_i}\left[\ell_i(\bfw;\bfx,\bfy)\right]$, $n=|\mathcal{N}|$ is the number of devices, and $\alpha$ is the stepsize. The advantages of this formulation are two-fold: 1) It gives a personalized solution that can capture any heterogeneity between the devices; 2) the meta-model can quickly adapt to new devices via slightly updating it with respect to their own data. Clearly, FML well fits edge learning cases, where edge devices have insufficient computing power and limited data samples.

Next, we review the standard FML algorithm in the literature.

\subsection{Standard Algorithm}
\label{subsec:standard_alg}
Similar to federated learning, vanilla FML algorithm solves \eqref{prob:fed_meta} in two repeating steps: \emph{local update} and \emph{global aggregation} \cite{fallah2020personalized}, as detailed below.
\begin{itemize}
	\item \emph{Local update:} At the beginning of each round $k$, the server first sends the current global model $\bfw^k$ to a fraction of devices $\mathcal{N}_k$ chosen uniformly at random with pre-set size $n_k$. Then, each device $i\in\mathcal{N}_k$ updates the received model based on its meta-function $F_i(\bfw) \coloneqq f_i(\bfw - \alpha\nabla f_i(\bfw))$ by running $\tau$ ($\ge1$) steps of stochastic gradient descent locally (also called mini-batch gradient descent), \ie,
	\begin{align}
		\label{eq:local_update}
		\bfw^{k,t+1}_i = \bfw^{k,t}_i - \beta\tilde{\nabla}F_i(\bfw^{k,t}_i),~\mathrm{for}~0\le t\le\tau-1
	\end{align}
	where $\bfw^k_t$ denotes the local model of device $i$ in the $t$-th step of the local update in round $k$ with $\bfw^{k,0}_i = \bfw^k$, and $\beta>0$ is the meta--learning rate. In \eqref{eq:local_update}, the stochastic gradient $\tilde{\nabla}F_i(\bfw)$ is given by 
	\begin{align}
		\label{eq:gradient_estimate}
		\tilde{\nabla}F_i(\bfw)\coloneqq \big(I - \alpha\tilde{\nabla}^2 f_i(\bfw,\mathcal{D}''_i)\big)\tilde{\nabla}f_i\big(\bfw - \alpha\tilde{\nabla} f_i(\bfw,\mathcal{D}_i),\mathcal{D}'_i\big)
	\end{align}
	where $\mathcal{D}_i$, $\mathcal{D}'_i$, and $\mathcal{D}''_i$ are independent batches\footnote{We slightly abuse the notation $\mathcal{D}_i$ as a batch of the local dataset of the $i$-th device.}, and for any batch $\mathcal{D}$, $\tilde{\nabla} f_i(\bfw,\mathcal{D})$ and $\tilde{\nabla}^2 f_i(\bfw,\mathcal{D})$ are the unbiased esimates of $\nabla f_i(\bfw)$ and $\nabla^2 f_i(\bfw)$ respectively, \ie,
	\begin{align}
		\tilde{\nabla}f_i(\bfw,\mathcal{D})&\coloneqq 	\frac{1}{|\mathcal{D}|}\sum_{(\bfx,\bfy)\in\mathcal{D}}\nabla\ell_i(\bfw;\bfx,\bfy)\\
		\tilde{\nabla}^2f_i(\bfw,\mathcal{D})&\coloneqq 	\frac{1}{|\mathcal{D}|}\sum_{(\bfx,\bfy)\in\mathcal{D}}\nabla^2\ell_i(\bfw;\bfx,\bfy).
	\end{align}
	
	\item \emph{Global aggregation:} After updating the local model parameter, each selected device sends its local model $\bfw^k_i=\bfw^{k,\tau-1}_i$ to the server. The server updates the global model by averaging over the received models, \ie,
	\begin{align}
		\label{eq:global_agregation}
		\bfw^{k+1}=\frac{1}{n_k}\sum_{i\in\mathcal{N}_k}\bfw^k_i.
	\end{align}         
\end{itemize}
It is easy to see that the main difference between federated learning and FML lies in the local update phase: In federated learning, local update is done using the unbiased gradient estimates while FML uses the biased one consisting of high-order information. Besides, federated learning can be considered as a special case of FML, \ie, FML under $\alpha=0$.

\subsection{Assumptions}
In this subsection, we list the standard assumptions for the analysis of FML algorithms \cite{fallah2020personalized,lin2020collaborative,yue2021inexact}.
\begin{assumption}[\textbf{Smoothness}]
	\label{assump:smoothness}
	The expected loss function $f_i$ corresponding to device $i\in\mathcal{N}$ is twice continuously differentiable and $L_i$-smooth, \ie,
	\begin{align}
		\quad \|\nabla f_i(\bfw_1)-\nabla f_i(\bfw_2)\|\le L_i\|\bfw_1 - \bfw_2\|,~ \forall \bfw_1,\bfw_2\in\mathbb{R}^d.
	\end{align}
	Besides, its gradient is bounded by a positive constant $\zeta_i $, \ie, $\|\nabla f_i(\bfw)\|\le \zeta_i$.
\end{assumption}
\begin{assumption}[\textbf{Lipschitz Hessian}]
	\label{assump:lipschitz_hessian}
	The Hessian of function $f_i$ is $\rho_i$-Lipschitz continuous for each $i\in\mathcal{N}$, \ie,
	\begin{align}
		\|\nabla^2 f_i(\bfw_1)-\nabla^2 f_i(\bfw_2)\|\le\rho_i\|\bfw_1-\bfw_2\|,~\forall \bfw_1,\bfw_2\in\mathbb{R}^d.
	\end{align}
\end{assumption}
\begin{assumption}[\textbf{Bounded Variance}]
	\label{assump:bounded_variance}
	Given any $\bfw\in\mathbb{R}^d$, the following facts hold for stochastic gradient $\nabla\ell_i (\bfw;\bfx,\bfy)$ and Hessians $\nabla^2\ell_i (\bfw;\bfx,\bfy)$ with $(\bfx,\bfy)\in\mathcal{X}\times\mathcal{Y}$
	\begin{align}
		\mathbb{E}_{(\bfx,\bfy)\sim P_i}\left[\left\|\nabla\ell_i (\bfw;\bfx,\bfy) - \nabla f_i(\bfw)\right\|^2\right] &\le \sigma^2_G\\
		\mathbb{E}_{(\bfx,\bfy)\sim P_i}\left[\left\|\nabla^2\ell_i (\bfw;\bfx,\bfy) - \nabla^2 f_i(\bfw)\right\|^2\right] &\le \sigma^2_H.
	\end{align}
\end{assumption}
\begin{assumption}[\textbf{Similarity}]
	\label{assump:similarity}
	For any $\bfw\in\mathbb{R}^d$ and $i,j\in\mathcal{N}$, there exist nonnegtive constants $\gamma_G\ge0$ and $\gamma_H\ge0$ such that the gradients and Hessians of the expected loss funtions $f_i(\bfw)$ and $f_j(\bfw)$ satisfy the following conditions
	\begin{align}
		\|\nabla f_i(\bfw)-\nabla f_j(\bfw)\|&\le\gamma_G\\
		\|\nabla^2 f_i(\bfw)-\nabla^2 f_j(\bfw)\|&\le\gamma_H.
	\end{align}
\end{assumption}
Assumption \ref{assump:lipschitz_hessian} implies the high-order smoothness of $f_i(\bfw)$ dealing with the second-order information in the local update step \eqref{eq:local_update}. Assumption \ref{assump:similarity} indicates that the variations of gradients between different devices are bounded by some constants, which captures the similarities between devices' tasks corresponding to non-IID data. It holds for many practical loss functions \cite{zhang2020fedpd}, such as logistic regression and hyperbolic tangent functions. In particular, $\psi^g_i$ and $\psi^h_i$ can be roughly seen as a distance between data distributions $P_i$ and $P_j$ \cite{fallah2020convergence}.

\section{Non-Uniform Federated Meta-Learning}
\label{sec:nufm}
Due to the uniform selection of devices in each round, the convergence rate of the standard FML algorithm is naturally slow. In this section, we present a non-uniform device selection scheme to tackle this challenge. 

\subsection{Device Contribution Quantification}
We begin with quantifying the contribution of each device to the reduction of one-round global loss using its dataset size and gradient norms. For convenience, define $\zeta\coloneqq\max_i \zeta_i $, $L\coloneqq\max_i L_i$, and $\rho\coloneqq\max_i \rho_i$. We first provide necessary lemmas before giving the main result.

\begin{lemma}
	\label{lem:smoothness_Fi}
	If Assumptions \ref{assump:smoothness} and \ref{assump:lipschitz_hessian} hold, local meta-function $F_i$ is smooth with parameter $L_F\coloneqq (1+\alpha L)^2L+\alpha\rho \zeta$.
\end{lemma}
\begin{proof}[Sketch of Proof]
We expand the expression of $\|\nabla F_i(\theta_1)-\nabla F_i(\theta_2)\|$ into two parts by triangle inequality, followed by bounding each part via Assumptions 1 and 2. The detailed proof is presented in Appendix \ref{proof:lem_smooth_Fi} of the technical report \cite{technicalreport}.
\end{proof}
Lemma \ref{lem:smoothness_Fi} gives the smoothness of the local meta-function $F_i$ and global loss function $F$. 

\begin{lemma}
	\label{lem:variance_Fi}
	Suppose that Assumptions \ref{assump:smoothness}-\ref{assump:bounded_variance} are satisfied, and $\mathcal{D}_i$, $\mathcal{D}'_i$, and $\mathcal{D}''_i$ are independent batches with sizes $D_i$, $D'_i$, and $D''_i$ respectively. For any $\bfw\in\mathbb{R}^d$, the following holds
	\begin{align}
		\left\| \mathbb{E}\left[\tilde{\nabla}F_i(\bfw) - \nabla F_i(\bfw)\right]\right\| &\le \frac{\alpha\sigma_G L(1+\alpha L)}{\sqrt{D_i}}\\
		\mathbb{E}\left[\left\| \tilde{\nabla}F_i(\bfw) - \nabla F_i(\bfw)\right\|^2\right]&\le \sigma^2_{F_i}
	\end{align}
	where $\sigma^2_{F_i}$ is denoted as
	\begin{align}
		\label{eq:sigma_Fi}
		\sigma^2_{F_i}\coloneqq&~ 6\sigma^2_G(1+\alpha L)^2\left(\frac{1}{D'_i}+\frac{(\alpha L)^2}{D_i}\right)+\frac{3(\alpha \zeta\sigma_H)^2}{D''_i}\nonumber\\
		&+ \frac{6(\alpha\sigma_G\sigma_H)^2}{D''_i}\left(\frac{1}{D'_i}+\frac{(\alpha L)^2}{D_i}\right).
	\end{align}
\end{lemma}
\begin{proof}[Sketch of Proof]
We first obtain
\begin{align*}
    \big\|\mathbb{E}\big[\tilde{\nabla}F_i(\bfw) &- \nabla F_i(\bfw)\big]\big\| \le \big\|\big(I - \alpha\nabla^2 f_i(\bfw)\big)\mathbb{E}\big[\delta^*_2\big] + \mathbb{E}\big[\delta^*_1\delta^*_2\big] \nonumber\\
    & + \mathbb{E}\big[\delta^*_1\big]\nabla f_i(\bfw - \alpha\nabla f_i(\bfw))\big\|,\\
    \mathbb{E}\big[\|\tilde{\nabla}F_i(\bfw) &- \nabla F_i(\bfw)\|^2\big]\le  3\Big(\mathbb{E}\big[\|\delta^*_1\|^2\big]\mathbb{E}\big[\|\delta^*_2\|^2\big] \nonumber\\
    & + \zeta^2\mathbb{E}\big[\|\delta^*_1\|^2\big] +  \left(1+\alpha L\right)^2\mathbb{E}\big[\|\delta^*_2\|^2\big]\Big),
\end{align*}
where $\delta^*_1$ and $\delta^*_2$ are given by
\begin{align*}
	\delta^*_1 &= \alpha\left(\nabla^2 f_i(\bfw) - \tilde{\nabla}^2 f_i(\bfw,\mathcal{D}''_i)\right)\\
	\delta^*_2 &= \tilde{\nabla} f_i\big(\bfw - \alpha\tilde{\nabla} f_i(\bfw,\mathcal{D}_i),\mathcal{D}'_i\big) - \nabla f_i\left(\bfw - \alpha\nabla f_i(\bfw)\right).
\end{align*}
Then, we derive the results via bounding the first and second moments of $\delta^*_1$ and $\delta^*_2$. The detailed proof is presented in Appendix \ref{proof:lem_variance_Fi} of the technical report \cite{technicalreport}.
\end{proof}
Lemma \ref{lem:variance_Fi} shows that the stochastic gradient $\tilde{\nabla}F_i(\bfw)$ is a biased estimate of $\nabla F_i(\bfw)$, revealing the challenges in analyzing FML algorithms.

\begin{lemma}
	\label{lem:similarity_Fi}
	If Assumptions \ref{assump:smoothness}, \ref{assump:lipschitz_hessian}, and \ref{assump:similarity} are satisfied, then for any $\bfw\in\mathbb{R}^d$ and $i,j\in\mathcal{N}$, we have
	\begin{align}
		\left\|\nabla F_i(\bfw) - \nabla F_j(\bfw)\right\|\le \left(1 + \alpha L\right)^2\gamma_G + \alpha \zeta \gamma_H.
	\end{align}
\end{lemma}
\begin{proof}[Sketch of Proof]
We divide the bound of $\|\nabla F_i(\bfw) - \nabla F_j(\bfw)\|$ into two independent terms, followed by bounding the two terms separately. The detailed proof is presented in Appendix \ref{proof:lem_similarity_Fi} of the technical report \cite{technicalreport}.
\end{proof}
Lemma \ref{lem:similarity_Fi} characterizes the similarities between the local meta-functions, which is critical for analyzing the one-step global loss reduction because it relates the local meta-function and global objective. Based on Lemmas \ref{lem:smoothness_Fi}--\ref{lem:similarity_Fi}, we are now ready to give our main result.

\begin{theorem}
	\label{thm:gap}
	Suppose that Assumptions \ref{assump:smoothness}-\ref{assump:similarity} are satisfied, and $\mathcal{D}_i$, $\mathcal{D}'_i$, and $\mathcal{D}''_i$ are independent batches. If the local update and global aggregation follow \eqref{eq:local_update} and \eqref{eq:global_agregation} respectively, then the following fact holds true for $\tau=1$
	\begin{align}
		\label{eq:thm1_1}
		\mathbb{E}[F(\bfw^k) - F(\bfw^{k+1})]\ge\;& \beta\mathbb{E}\Bigg[\frac{1}{n_k}\sum_{i\in\mathcal{N}_k}\Bigg(\Big(1-\frac{L_F\beta}{2}\Big)\big\|\tilde{\nabla} F_i(\bfw^k)\big\|^2\nonumber\nonumber\\
		&-\bigg(\sqrt{\left(1 + \alpha L\right)^2\gamma_G + \alpha \zeta \gamma_H}+\sigma_{F_i}\bigg)\nonumber\\
		&\times\sqrt{ \mathbb{E}\Big[\big\|\tilde{\nabla} F_i(\bfw^k)\big\|^2\Big | \mathcal{N}_k\Big]}\Bigg)\Bigg]
	\end{align}
	where the outer expectation of RHS is taken with respect to the selected user set $\mathcal{N}_k$ and data sample sizes, and the inner expectation is only regarding data sample sizes.
\end{theorem}
\begin{proof}[Sketch of Proof]
Using the smoothness condition of $F_i$, we express the lower bound of loss reduction by 
\begin{align*}
	\mathbb{E}[F(\bfw^k) - F(\bfw^{k+1})]\ge \mathbb{E}[G^k],
\end{align*}
where $G^k$ is defined as
\begin{align*}
	G^k\coloneqq\;&\beta\nabla F(\bfw^k)^\top\left(\frac{1}{n_k}\sum_{i\in\mathcal{N}_k}\tilde{\nabla}F_i(\bfw^k)\right)\\
	&- \frac{L_F\beta^2}{2}\left\|\frac{1}{n_k}\sum_{i\in\mathcal{N}_k}\tilde{\nabla}F_i(\bfw^k)\right\|^2.
\end{align*}
The key step to derive the desired result is providing a tight lower bound for the product of $\nabla F(\bfw^k)$ and $\tilde{\nabla} F(\bfw^k)$. The detailed proof is presented in Appendix \ref{proof:thm:gap} of the technical report \cite{technicalreport}.
\end{proof}
Theorem \ref{thm:gap} provides a lower bound on the one-round reduction of the global objective function $F$ based on the device selection. It implies that different user selection has varying impacts on the objective improvement and quantifies the contribution of each device to the objective improvement, depending on the variance of local meta-function, task similarities, smoothness, and learning rates. It therefore provides a criterion for selecting users to accelerate the convergence. 

From Theorem \ref{thm:gap}, we have the following Corollary, which simplifies the above result and extends it to multi-step cases.
\begin{corollary}
	\label{coro:gap}
	Suppose that Assumptions \ref{assump:smoothness}-\ref{assump:similarity} are satisfied, and $\mathcal{D}_i$, $\mathcal{D}'_i$, and $\mathcal{D}''_i$ are independent batches with $D_i=D'_i=D''_i$. If the local update and global aggregation follow \eqref{eq:local_update} and \eqref{eq:global_agregation} respectively, then the following fact holds true with $\beta\in[0,1/L_F)$
	\begin{align}
		\label{eq:lowerbound}
		\mathbb{E}\big[F(\bfw^k) - F(\bfw^{k+1})\big]
		\ge \frac{\beta}{2}\mathbb{E}\Bigg[\frac{1}{n_k}\sum_{i\in\mathcal{N}_k}\sum^{\tau-1}_{t=0}\Bigg(\big\|\tilde{\nabla} F_i(\bfw^{k,t}_i)\big\|^2 \nonumber\\
		- 2\big(\lambda_1+\frac{\lambda_2}{\sqrt{D_i}}\big)\sqrt{ \mathbb{E}\Big[\big\|\tilde{\nabla} F_i(\bfw^{k,t}_i)\big\|^2\Big | \mathcal{N}_k\Big]}\Bigg)\Bigg]
	\end{align}
	where positive constants $\lambda_1$ and $\lambda_2$ satisfy that
	\begin{align}
		\label{eq:lambda_1}
		\lambda_1\ge&~ \sqrt{\left(1 + \alpha L\right)^2\gamma_G + \alpha \zeta \gamma_H} + \beta\tau\sqrt{35 (\gamma^2_G + 2\sigma^2_F)}\\
		\label{eq:lambda_2}
		\lambda^2_2\ge&~
		6\sigma^2_G\left(1+(\alpha L)^2\right)\left((\alpha\sigma_H)^2+(1+\alpha L)^2\right)\nonumber\\
		&+3(\alpha\zeta\sigma_H)^2.
	\end{align}
\end{corollary}
\begin{proof}[Sketch of Proof]
The proof is similar to Theorem \ref{thm:gap}, with additional tricks in bounding the product of $\nabla F(\bfw^k)$ and $\tilde{\nabla} F(\bfw^k)$.  The detailed proof is presented in Appendix \ref{proof:corollaries} of the technical report \cite{technicalreport}.
\end{proof}
Corollary \ref{coro:gap} implies that the device with a large gradient naturally accelerates the global loss decrease, but a small dataset size degrades the process due to corresponding high variance. Besides, as the device dissimilarities become large, the lower bound \eqref{eq:lowerbound} weakens. 

Motivated by Corollary \ref{coro:gap}, we study the device selection in the following subsection.

\subsection{Device Selection}
To improve the convergence speed, we aim to maximize the lower bound \eqref{eq:lowerbound} on the one-round objective reduction. Based on Corollary \ref{coro:gap}, we define the contribution $u^k_i$ device $i$ to the convergence in round $k$ as 
\begin{align}
	\label{eq:ui}
	u^k_i \coloneqq \sum^{\tau-1}_{t=0}\big\|\tilde{\nabla} F_i(\bfw^{k,t}_i)\big\|^2 - 2\big(\lambda_1+\frac{\lambda_2}{\sqrt{D_i}}\big)\big\|\tilde{\nabla} F_i(\bfw^{k,t}_i)\big\|
\end{align}
where we replace the second moment of $\tilde{\nabla} F_i(\bfw^{k,t}_i)$ by its sample value $\|\tilde{\nabla} F_i(\bfw^{k,t}_i)\|^2$. Then, the device selection problem in round $k$ can be formulated as
\begin{align}
	\label{prob:selection}
	\begin{split}
		\max_{\{z_i\}}~&\sum_{i\in\mathcal{N}}z_i u^k_i\\
		\mathrm{s.t.}~& \sum_{i\in\mathcal{N}}z_i= n_k\\
		&\, z_i\in\{0,1\},~\forall i\in\mathcal{N}.
	\end{split}
\end{align}
In \eqref{prob:selection}, $z_i$ is a binary variable: $z_i=1$ for selecting device $i$ in this round; $z_i=0$, otherwise. The solution of \eqref{prob:selection} can be found by the \emph{Select} algorithm introduced in \cite[Chapter 9]{cormen2009introduction} with the worst-case complexity $O(n)$ (the problem \eqref{prob:selection} is indeed a \emph{selection problem}). Accordingly, our device selection scheme is presented as follows.

\emph{Device selection:} Following the local update phase, instead of selecting devices uniformly at random, each device $i$ first computes its contribution scalar $u^k_i$ locally and sends it to the server. After receiving $\{u^k_i\}_{i\in\mathcal{N}}$ from all devices, the server runs the \emph{Select} algorithm and finds the optimal device set denoted by $\mathcal{N}^*_k$. Notably, although constants $\lambda_1$ and $\lambda_2$ in \eqref{prob:selection} consist of unknown parameters such as $L$, $\gamma_G$, and $\gamma_H$, they can be either estimated during training as in \cite{wang2019adaptive} or directly tuned as in our simulations.

Based on the device selection scheme, we propose the \emph{Non-Uniform Federated Meta-Learning (\algname{NUFM})} algorithm (depicted in Algorithm \ref{alg:nufm}). In particular, although \algname{NUFM} requires an additional communication phase to upload $u^k_i$ to the server, the communication overhead can be negligible because $u^k_i$ is just a scalar. 

\begin{algorithm}[ht]
	\caption{Non-Uniform Federated Meta-Learning (\algname{NUFM})}
	\label{alg:nufm}
	\LinesNumbered
	\KwIn{$\alpha$, $\beta$, $\lambda_1$, $\lambda_2$}
	Server initializes model $\bfw^0$ and sends it to all devices\;
	\For{round $k=0$ \KwTo $K-1$}{
		\ForEach{device $i\in\mathcal{N}$}{
			\tcp{Local update}
			Initialize $\bfw^{k,0}_i \leftarrow \bfw^k$ and $u^k_i \leftarrow 0$\;
			\For{local step $t=0$ \KwTo $\tau-1$}{
				Compute stochastic gradient $\tilde{\nabla}F_i(\bfw^{k,t}_i)$ by \eqref{eq:gradient_estimate} using batches $\mathcal{D}_i$, $\mathcal{D}'_i$, and $\mathcal{D}''_i$\;
				Update local model $\bfw^{k,t+1}_i$ by \eqref{eq:local_update}\;
				Update contribution scalar $u^k_i$ by $\begin{aligned}[b]
					u^k_i \leftarrow&~ u^k_i + \|\tilde{\nabla} F_i(\bfw^{k,t}_i)\|^2\\ &- 2(\lambda_1+\tfrac{\lambda_2}{\sqrt{D_i}})\|\tilde{\nabla} F_i(\bfw^{k,t}_i)\|
				\end{aligned}$\;	
			}
			Set $\bfw^k_i = \bfw^{k,\tau}_i$ and send $u^k_i$ to server\;
		}
		\tcp{Device selection}
		Once receiving $\{u^k_i\}_{i\in\mathcal{N}}$, server computes optimal device selection $\mathcal{N}^*_k$ by solving \eqref{prob:selection}\;
		\tcp{Global aggregation}
		After receiving local models $\{\bfw^k_i\}_{i\in\mathcal{N}^*_k}$, server computes the global model by \eqref{eq:global_agregation}\;
	}
	\KwRet{$\bfw^K$}
\end{algorithm} 

\section{Federated Meta-Learning over Wireless Networks}
\label{sec:wireless_nufm}
In this section, we extend \algname{NUFM} to the context of multi-access wireless systems, where the bandwidth for uplink transmission and the power of IoT devices are limited. First, we present the system model followed by the problem formulation. Then, we decompose the original problem into two sub-problems and devise solutions for each of them with theoretical performance guarantees.

\subsection{System Model}
As illustrated in Fig. \ref{figure:system}, we consider a wireless multi-user system, where a set $\mathcal{N}$ of $n$ end devices joint forces to carry out federated meta-learning aided by an edge server. Each round consists of two stages: the computation phase and the communication phase. In the computation phase, each device $i\in\mathcal{N}$ downloads the current global model and computes the local model based on its local dataset; in the communication phase, the selected devices transmit the local models to the edge server via a limited number of wireless channels. After that, the edge server runs the global aggregation and starts the next round. Here we do not consider the downlink communication due to the asymmetric uplink-downlink settings in wireless networks. That is, the transmission power at the server (\eg, a base station) and the downlink communication bandwidth are generally sufficient for global meta-model transmission. Thus, the downlink time is usually neglected, compared to the uplink data transmission time \cite{dinh2020federated}. Since we focus on the device selection and resource allocation problem in each round, we omit subscript $k$ for brevity throughout this section.

\begin{figure}[ht]
	\centering
	\includegraphics[width=0.9\linewidth]{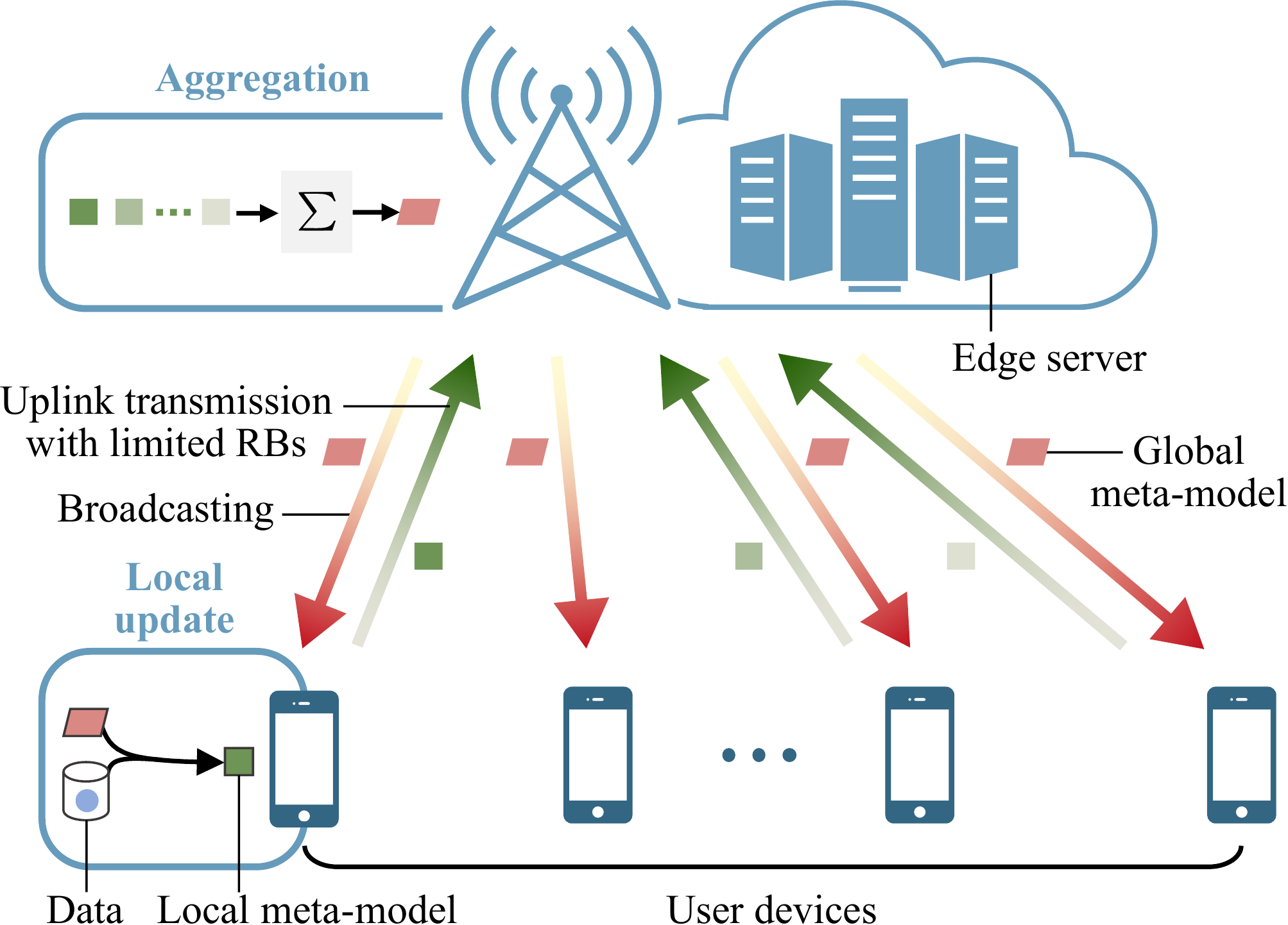}
	\caption{The architecture of federated meta-learning over a wireless network with multiple user devices and an edge server. Due to limited communication resources, only part of user devices can upload their local models in each training round.}
	\label{figure:system}
\end{figure}

\subsubsection{Computation Model}
We denote $c_i$ as the CPU cycles for device $i$ to update the model with one sample, which can be measured offline as a priori knowledge \cite{dinh2020federated}. Assume that the batch size of device $i$ used in local update phase \eqref{eq:local_update} is $D_i$. Then, the number of CPU cycles required for device $i$ to run a one-step local update is $c_i D_i$. We denote the CPU-cycle frequency of device $i$ as $\nu_i$. Thus, the CPU energy consumption of device $i$ in the computation during the local update phase can be expressed by
\begin{align}
	E^\mli{cp}_i(\nu_i) \coloneqq \frac{\iota_i}{2}\tau_i c_i D_i \nu^2_i
\end{align}
where $\iota_i/2$ is the effective capacitance coefficient of the computing chipset of device $i$ \cite{burd1996processor}. The computational time of device $i$ in a round can be denoted as
\begin{align}
	T^\mli{cp}_i(\nu_i)\coloneqq\frac{\tau_i c_i D_i}{\nu_i}.
\end{align}
For simplicity, we set $\tau=1$ in the following. 

\subsubsection{Communication Model} 
We consider a multi-access protocol for devices, \ie, the orthogonal frequency division multiple access (OFDMA) technique whereby each device can occupy one uplink resource block (RB) in a communication round to upload its local model. There are $M$ RBs in the system, denoted by $\mathcal{M}=\{1,2,\dots,M\}$. The achievable transmission rate of device $i$ is \cite{chen2020convergence}
\begin{align}
	\label{eq:transmission_rate}
	r_i(\bm{z}_i,p_i) \coloneqq \sum_{m\in\mathcal{M}}z_{i,m}B\log_2\left(1+\frac{h_i p_i}{I_{m} + B N_0}\right)
\end{align}
with $B$ being the bandwidth of each RB, $h_i$ the channel gain, $N_0$ the noise power spectral density, $p_i$ the transmission power of device $i$, and $I_{m}$ the interference caused by the devices that are located in other service areas and use the same RB. In \eqref{eq:transmission_rate}, $z_{i,m}\in\{0,1\}$ is a binary variable associated with the $m$-th RB allocation for device $i$: $z_{i,m}=1$ indicates that RB $m$ is allocated to device $i$, and $z_{i,m}=0$ otherwise. Each device can only occupy one RB at maximum while each RB can be accessed by at most one device, thereby we have
\begin{align}
	\label{eq:constraint_zm}
	\sum_{m\in\mathcal{M}}z_{i,m}&\le 1\quad\forall i\in\mathcal{N}\\
	\label{eq:constraint_zi}
	\sum_{i\in\mathcal{N}}z_{i,m}&\le 1\quad\forall m\in\mathcal{M}.
\end{align}

Due to the fixed dimension of model parameters, we assume that the model sizes of devices are constant throughout the learning process, denoted by $S$. If device $i$ is selected, the time duration of transmitting the model is given by
\begin{align}
	\label{eq:Tico}
	T^\mli{co}_i(\bm{z}_i,p_i) \coloneqq \frac{S}{r_i(\bm{z}_i,p_i)}
\end{align}
where $\bm{z}_i\coloneqq \{z_{i,m}\mid m\in\mathcal{M}\}$. Besides, the energy consumption of the transmission is
\begin{align}
	E^\mli{co}_i(\bm{z}_i,p_i) \coloneqq \sum_{m\in\mathcal{M}}z_{i,m}T^\mli{co}_i(\bm{z}_i,p_i) p_i.
\end{align}
If no RB is allocated to device $i$ in current round, its transmission power and energy consumption is zero.

\subsection{Problem Formulation}
For ease of exposition, we define $\bm{\nu}\coloneqq \{\nu_i\mid i\in\mathcal{N}\}$, $\bm{p}\coloneqq \{p_i\mid i\in\mathcal{N}\}$, and $\bm{z}\coloneqq\{z_{i,m}\mid i\in\mathcal{N},m\in\mathcal{M}\}$. Recall the procedure of \algname{NUFM}. The total energy consumption $E(\bm{z},\bm{p},\bm{\nu})$ and wall-clock time $T(\bm{z},\bm{p},\bm{\nu})$ in a round can be expressed by
\begin{align}
	\label{eq:E}
	E(\bm{z},\bm{p},\bm{\nu}) &\coloneqq \sum_{i\in\mathcal{I}}\Big(E^\mli{cp}_i(\nu_i) + E^\mli{co}_i(\bm{z}_i,p_i)\Big)\\
	\label{eq:T}
	T(\bm{z},\bm{p},\bm{\nu}) &\coloneqq \max_{i\in\mathcal{N}} T^\mli{cp}_i(\nu_i) + \max_{i\in\mathcal{N}} \sum_{m\in\mathcal{M}} z_{i,m}T^\mli{co}_i(\bm{z}_i,p_i)
\end{align}
where we neglect the communication time for transmitting the scalar $u_i$. The total contribution to the convergence is
\begin{align}
	\label{eq:U}
	U(\bm{z}) = \sum_{i\in\mathcal{N}}\sum_{m\in\mathcal{M}}z_{i,m}u_i
\end{align}
where $u_i$ is given in \eqref{eq:ui}\footnote{One can regularize $u_i$ via adding a large enough constant in \eqref{eq:ui} to keep it positive.}.

We consider the following non-convex mixed-integer non-linear programming (MINLP) problem
\begin{align}
	\mathrm{(P)}~\mathop{\max}_{\bm{z},\bm{p},\bm{\nu}}&~~U(\bm{z}) - \eta_1 E(\bm{z},\bm{p},\bm{\nu}) - \eta_2 T(\bm{z},\bm{p},\bm{\nu})\nonumber\\
	\mathrm{s.t.}\label{eq:constraints_pi}
	&~~0\le p_i\le p^\mli{max}_i,~\forall i\in\mathcal{N}\\
	\label{eq:constraints_nui}
	&~~0\le\nu_i\le\nu^\mli{max}_i,~\forall i\in\mathcal{N}\\
	\label{eq:constraint_z01}
	&~~z_{i,m}\in\{0,1\},~\forall i\in\mathcal{N},m\in\mathcal{M}\\
	&~\,\sum\nolimits_{i\in\mathcal{N}}z_{i,m}\le 1,~\forall m\in\mathcal{M}\tag{\ref{eq:constraint_zm}}\\
	&~\,\sum\nolimits_{m\in\mathcal{M}}z_{i,m}\le 1,~ \forall i\in\mathcal{N}\tag{\ref{eq:constraint_zi}}
\end{align}
where $\eta_1\ge0$ and $\eta_2\ge0$ are weight parameters to capture the Pareto-optimal trade-offs among convergence, latency, and energy consumption, the values of which depend on specific scenarios. Constraints \eqref{eq:constraints_pi} and \eqref{eq:constraints_nui} give the feasible regions of devices' transmission power levels and CPU-cycle frequencies, respectively. Constraints \eqref{eq:constraint_zm} and \eqref{eq:constraint_zi} restrict that each device can only access one uplink RB while each RB can be allocated to one device at most.

In this formulation, we aim to maximize the convergence speed of FML, while minimizing the energy consumption and wall-clock time in each round. Notably, our solution can adapt to the problem with hard constraints on energy consumption and wall-clock time as in \cite{chen2020joint} via setting ``virtual devices'' (see Lemma \ref{lem:sp1_nui} and Lemma \ref{lem:sp2_pi}). 

Next, we provide a joint device selection and resource allocation algorithm to solve this problem.

\subsection{A Joint Device Selection and Resource Allocation Algorithm}
Substituting \eqref{eq:E}, \eqref{eq:T}, and \eqref{eq:U} into problem (P), we can easily decompose the original problem into the following two sub-problems (SP1) and (SP2).
\begin{alignat*}{2}
	\mathrm{(SP1)}~&\mathop{\min}_{\bm{\nu}}&&~~ g_1(\bm{\nu})=\eta_1\sum_{i\in\mathcal{N}}\frac{\iota_i}{2}c_i D_i \nu^2_i + \eta_2\max_{i\in\mathcal{N}}\frac{c_i D_i}{\nu_i}\\
	~&~\,\mathrm{s.t.}&&~~ 0 \le \nu_i \le \nu^\mli{max}_i,~\forall i\in\mathcal{N}.\\
	\mathrm{(SP2)}~&\mathop{\max}_{\bm{z},\bm{p}}&&~~ g_2(\bm{z},\bm{p})= \sum_{i\in\mathcal{N}}\sum_{m\in\mathcal{M}}z_{i,m}u_i\\
	&~ &&~~-\sum_{i\in\mathcal{N}}\sum_{m\in\mathcal{M}}\eta_1\frac{S p_i}{B\log_2\left(1+\frac{h_i p_i}{I_{m} + BN_0}\right)}\\
	&~ &&~~ -\eta_2\max_{i\in\mathcal{N}}\sum_{m\in\mathcal{M}}z_{i,m}\frac{S}{B\log_2\left(1+\frac{h_i p_i}{I_{m} + BN_0}\right)}\\
	~&~\,\mathrm{s.t.} &&~~0 \le p_i \le p^\mli{max}_i,~\forall i\in\mathcal{N}\\
	~&~ &&~\,\sum\nolimits_{i\in\mathcal{N}}z_{i,m}\le 1,~\forall m\in\mathcal{M}\\
	~&~ &&~\, \sum\nolimits_{m\in\mathcal{M}}z_{i,m}\le 1,~ \forall i\in\mathcal{N}\\
	~&~ &&~~z_{i,m}\in\{0,1\},~\forall i\in\mathcal{N},m\in\mathcal{M}.
\end{alignat*}
(SP1) aims at controlling the CPU-cycle frequencies for devices to minimize the energy consumption and latency in the computational phase. (SP2) controls the transmission power and RB allocation to maximize the convergence speed while minimizing the transmission cost and communication delay. We provide the solutions to these two sub-problems separately.

\subsubsection{Solution to (SP1)}
\label{sec:solution_sp1} 
Denote the optimal CPU-cycle frequencies of (SP1) as $\bm{\nu}^* = \{\nu^*_i\}_{i\in\mathcal{N}}$. We first give the following lemma to offer insights into this sub-problem.
\begin{lemma}
	\label{lem:sp1_nui}
	If device $j$ is the straggler of all devices with optimal CPU frequencies $\bm{\nu}^*$, \ie, $j=\mathop{\arg\max}_{i\in\mathcal{N}}c_i D_i/\nu^*_i$, the following holds true for any $\eta_1,\eta_2>0$
	\begin{align}
		\label{eq:optimal_j}
		\nu^*_i=
		\begin{cases}
			\min\left\{\sqrt[3]{\frac{a_2}{2a_1}},\min_{i\in\mathcal{N}}\frac{c_j D_j \nu^\mli{max}_j}{c_i D_i}\right\}&,\mathrm{if}~i=j\\
			\frac{c_i D_i \nu^*_j}{c_j D_j}&,\mathrm{otherwise}.
		\end{cases}
	\end{align}
	The positive constants $a_1$ and $a_2$  in \eqref{eq:optimal_j} are defined as
	\begin{align}
		a_1&\coloneqq \eta_1\sum_{i\in\mathcal{N}}\frac{\iota_i (c_i D_i)^3}{2(c_j D_j)^2}\\
		a_2&\coloneqq \eta_2c_j D_j.
	\end{align}
\end{lemma}
\begin{proof}[Sketch of Proof]
The derivation of the results involves two steps, \ie, expressing $v^*_i$ by $v^*_j$ and deriving $v^*_j$ by solving the corresponding optimization problem. The detailed proof is presented in Appendix \ref{proof:lem_sp1_nui} of the technical report \cite{technicalreport}.
\end{proof}
Lemma \ref{lem:sp1_nui} implies that if the straggler (a device with the lowest computational time) can be determined, then the optimal CPU-cycle frequencies of all devices can be derived as closed-form solutions. Intuitively, due to the contradiction in minimizing the energy consumption and computational time, if the straggler is fixed, then the other devices can use the smallest CPU-cycle frequencies as long as the computational time is shorter than that of the straggler. It leads to the following Theorem. 
\begin{theorem}
	\label{thm:find_nu}
	Denote $\bm{\nu}^*_{\mli{straggler:}j}$ as the optimal solution (\ie, \eqref{eq:optimal_j}) under the assumption that $j$ is the straggler. Then, the global optimal solution of (SP1) can be obtained by
	\begin{align}
		\label{eq:find_nu}
		\bm{\nu}^*=\mathop{\arg\min}_{\bm{\nu}\in\mathcal{V}}g_1(\bm{\nu})
	\end{align}
	where $\mathcal{V}\coloneqq\{\bm{\nu}^*_{\mli{straggler:}j}\}_{j\in\mathcal{N}}$, and $g_1$ is the objective function in (SP1).
\end{theorem}
\begin{proof}
The result can be directly obtained from Lemma \ref{lem:sp1_nui}. We omit it for brevity.
\end{proof}
Theorem \ref{thm:find_nu} shows that the optimal solution of (SP1) is the fixed-straggler solution in Lemma \ref{lem:sp1_nui} corresponding to the minimum objective $g_1$. Thus, (SP1) can be solved with computational complexity $O(n)$ by comparing the achievable objective values corresponding to different stragglers.

\subsubsection{Solution to (SP2)}
Similar to Section~\ref{sec:solution_sp1}, we denote the optimal solutions of (SP2) as $\bm{z}^*=\{z^*_{i,m}\mid i\in\mathcal{N},m\in\mathcal{M}\}$ and  $\bm{p}^*=\{p^*_i\mid i\in\mathcal{N}\}$ respectively, $\mathcal{N}^*\coloneqq\{i\in\mathcal{N}\mid \sum_{m\in\mathcal{N}}z^*_{i,m}=1\}$ as the optimal set of selected devices and, for each $i\in\mathcal{N}^*$,  RB block allocated to $i$ as $m^*_i$, \ie, $z^*_{i,m^*_i}=1$. 

It is challenging to derive a closed-form solution for (SP2) because it is a non-convex MINLP problem with non-differentiable ``max'' operator in the objective. Thus, in the following, we develop an iterative algorithm to solve this problem and show that the algorithm will converge to a local minimum.

We begin with analyzing the properties of the optimal solution in the next lemma. 
\begin{lemma}
	\label{lemma:delta}
	Denote the transmission delay regarding $\bm{z}^*$ and $\bm{p}^*$ as $\delta^*$, \ie,
	\begin{align}
		\label{eq:definition_Delta}
		\delta^*\coloneqq\max_{i\in\mathcal{N}}\sum_{m\in\mathcal{M}}z^*_{i,m}\frac{S}{B\log_2\left(1+\frac{h_i p^*_i}{I_{m} + BN_0}\right)}.
	\end{align}
	The following relation holds
	\begin{align}
		\label{eq:Delta_3}
		& \delta^*\ge\frac{S}{B\log_2\left(1+\frac{h_i p^\mli{max}_i}{I_{m} + BN_0}\right)}
	\end{align}
	and $p^*_i$ can be expressed by
	\begin{align}
		\label{eq:Delta_2}
		p^*_i =
		\begin{cases}
			\frac{(I_{m^*_i}+BN_0)(2^\frac{S}{B\delta^*}-1)}{h_i}&,~\mathrm{if}~i\in\mathcal{N}^*\\
			0&,~\mathrm{otherwise}.
		\end{cases}
	\end{align}
\end{lemma}
\begin{proof}[Sketch of Proof]
We prove \eqref{eq:Delta_3} by contradiction and \eqref{eq:Delta_2} by solving the corresponding transformed problem. The detailed proof is presented in Appendix \ref{proof:lem_delta} of the technical report \cite{technicalreport}.
\end{proof}
Lemma \ref{lemma:delta} indicates that the optimal transmission power can be derived as a closed-form via \eqref{eq:Delta_2}, given the RB allocation and transmission delay. Lemma \ref{lemma:delta} also implies that for any RB allocation strategy $\tilde{\bm{z}}$ and transmission delay $\tilde{\delta}$ (not necessarily optimal), equation \eqref{eq:Delta_2} provides the ``optimal'' transmission power under $\tilde{\bm{z}}$ and $\tilde{\delta}$ as long as \eqref{eq:Delta_3} is satisfied. Based on that, we have the following result.
\begin{theorem}
	\label{thm:find_z}
	Denote $\mu_{i,m}\coloneqq(I_m+BN_0)(2^\frac{S}{B\delta^*}-1)/h_i$. Given transmission delay $\delta^*$, the optimal RB allocation strategy can be obtained by
	\begin{align}
		\label{eq:find_z}
		\bm{z}^*=\mathop{\arg\max}_{\bm{z}}\bigg\{\sum_{i,m}z_{i,m}\left(u_i- e_{i,m}\right),~\mathrm{s.t.}~\eqref{eq:constraint_zm}-\eqref{eq:constraint_z01}\bigg\}
	\end{align}
	where 
	\begin{align}
		e_{i,m}\coloneqq
		\begin{cases}
			\eta_1\delta^*\mu_{i,m}&,~\mathrm{if}~\mu_{i,m}\le p^\mli{max}_i\\
			u_i + 1&,~\mathrm{otherwise}.
		\end{cases}
	\end{align}
\end{theorem}
\begin{proof}
From Lemma \ref{lemma:delta}, Eq. \eqref{eq:find_z} holds if $\mu_{i,m}\le p^\mli{max}_i$. On the other hand, when $\mu_{i,m}> p^\mli{max}_i$, if device $i$ is selected, the transmission delay will be larger than $\delta^*$ (see the proof of Lemma \ref{lemma:delta}), which is contradictory to the given condition. Thus, when $\mu_{i,m}> p^\mli{max}_i$, we set $e_{i,m}=u_i + 1$, ensuring device $i$ not to be selected.
\end{proof}
Theorem \ref{thm:find_z} shows the optimal RB allocation strategy can be obtained by solving \eqref{eq:find_z}, given transmission delay $\delta^*$. Naturally, problem \eqref{eq:find_z} can be equivalently transformed to a bipartite matching problem. Consider a \emph{Bipartite Graph} $\mathcal{G}$ with source set $\mathcal{N}$ and destination set $\mathcal{M}$. For each $i\in\mathcal{N}$ and $m\in\mathcal{M}$, denote the weight of the edge from node $i$ to node $j$ as $w_{i\rightarrow j}$: If $u_i- e_{i,m}>0$, $w_{i\rightarrow j}=e_{i,m}-u_i$; otherwise, $w_{i\rightarrow j}=\infty$. Therefore, maximizing \eqref{eq:find_z} is equivalent to finding a matching in $\mathcal{G}$ with the minimum sum of weights. It means that we can obtain the optimal RB allocation strategy under fixed transmission delay via \emph{Kuhn-Munkres} algorithm with the worst complexity of $O(Mn^2)$ \cite{weisstein2011hungarian}.

We proceed to show how to iteratively approximate the optimal $\delta^*$, $\bm{p}^*$, and $\bm{z}^*$. 
\begin{lemma}
	\label{lem:sp2_pi}
	Let $j$ denote the communication straggler among all selected devices with respect to RB allocation $\bm{z}^*$ and transmission power $\bm{p}^*$, \ie, for any $i\in\mathcal{N}^*$, 
	\begin{align}
		\label{eq:straggler_co}
		&\sum_{m\in\mathcal{M}}z^*_{i,m}\underbrace{\frac{S}{B\log_2\left(1+\frac{h_i p^*_i}{I_{m} + BN_0}\right)}}_{T^\mli{co}_i(\bm{z}^*_i,p^*_i)}\nonumber\\
		&\le \sum_{m\in\mathcal{M}}z^*_{j,m}\underbrace{\frac{S}{B\log_2\left(1+\frac{h_j p^*_j}{I_{m} + BN_0}\right)}}_{T^\mli{co}_j(\bm{z}^*_j,p^*_j)}.
	\end{align}
	Then, the following holds true:
	\begin{enumerate}[label=\arabic*),leftmargin=*]
		\item Define function $f_4(p) \coloneqq b_1\big((1+p)\log_2(1+p)\ln2-p\big)-\eta_2$. Then, $f_4(p)$ is monotonically increasing with respect to $p\ge0$, and has unique zero point $\tilde{p}^0_j\in(0,b_2]$, where $b_1$ and $b_2$ are denoted as follows
		\begin{align}
		\label{eq:b1_lem6}
	   	b_1&\coloneqq\eta_1\sum_{i\in\mathcal{N}^*}\frac{I_{m^*_i} + BN_0}{h_i}\\
		\label{eq:b2_lem6}	b_2&\coloneqq2^{(1+\sqrt{\max\{\frac{\eta_2}{b_1},1\}-1})/\ln2};
		\end{align}
		\item Denote $(\mli{SNR})_j\coloneqq(I_{m^*_j}+BN_0)/h_j$. For $i\in\mathcal{N}^*$, we have
		\begin{align}
			\label{eq:optimal_pi}
			p^*_i=
			\begin{cases}
				\min\left\{\tilde{p}^0_j,\min_{i\in\mathcal{N}^*}\frac{h_i p^\mli{max}_i}{I_{m^*_i} + BN_0}\right\}&,\mathrm{if}~i=j\\
				\frac{(\mli{SNR})_i}{(\mli{SNR})_j}p^*_j&,\mathrm{otherwise}.
			\end{cases}
		\end{align}
	\end{enumerate}
\end{lemma}
\begin{proof}[Sketch of Proof]
We obtain the first result by analyzing the property of $f_4$ and derive \eqref{eq:optimal_pi} via solving the corresponding optimization problem. The detailed proof is presented in Appendix \ref{proof:lem_sp2_pi} of the technical report \cite{technicalreport}.
\end{proof}
Lemma \ref{lem:sp2_pi} indicates that, given optimal RB allocation strategy $\bm{z}^*$ and straggler, the optimal transmission power can be derived by \eqref{eq:optimal_pi}, different from Lemma \ref{lemma:delta} that requires the corresponding transmission delay $\delta^*$. Notably, in \eqref{eq:optimal_pi}, we can obtain zero point $\tilde{p}^0_j$ of $f_4$ with any required tolerance $\epsilon$ by \emph{Bisection method} in $\log_2(\frac{b_2}{\epsilon})$ iterations at most.

Similar to Theorem \ref{thm:find_nu}, we can find the optimal transmission power by the following theorem, given the RB allocation.
\begin{theorem}
	\label{thm:find_pi}
	Denote $\bm{p}^*_{\mli{straggler:}j}$ as the optimal solution under the assumption that $j$ is the communication straggler, given fixed RB allocation $\bm{z}^*$. The corresponding optimal transmission power is given by 
	\begin{align}
		\label{eq:find_optimal_p}
		\bm{p}^*=\mathop{\arg\max}_{\bm{p}\in\mathcal{P}}g_2(\bm{z}^*,\bm{p})
	\end{align}
	where $\mathcal{P}\coloneqq\{\bm{p}^*_{\mli{straggler:}j}\}_{j\in\mathcal{N}}$ and $g_2$ is the objective function defined in (SP2).
\end{theorem}
\begin{proof}
We can easily obtain the result from Lemma \ref{lem:sp2_pi}, thereby omitting it for brevity.
\end{proof}

Define the communication time corresponding to $\bm{z}$ and $\bm{p}$ as $T^\mli{co}(\bm{z},\bm{p})\coloneqq\max_{i\in\mathcal{N}} \sum_{m\in\mathcal{M}} z_{i,m}T^\mli{co}_i(\bm{z}_i,p_i)$. Based on Theorem \ref{thm:find_z} and \ref{thm:find_pi}, we have the following \emph{\underline{I}terati\underline{ve} \underline{S}olution (IVES)} algorithm to solve (SP2).

\emph{IVES:} We initialize transmission delay $\delta^0$ (based on \eqref{eq:Delta_2}) as follows
\begin{align}
	\label{eq:delta_0}
	\delta^0 = \max_{i\in\mathcal{N}}\frac{S}{B\log_2\left(1+\frac{h_i p^\mli{max}_i}{I_{m} + BN_0}\right)}.
\end{align} 
In each iteration $t$, we first compute an RB allocation strategy $\bm{z}^t$ via solving \eqref{eq:find_z} by the \emph{Kuhn-Munkres} algorithm. Then, based on $\bm{z}^t$, we find the corresponding transmission power $\bm{p}^t$ by \eqref{eq:find_optimal_p} and update the transmission delay by $\delta^{t+1}=T^\mli{co}(\bm{z}^t,\bm{p}^t)$ before the next iteration. The details of IVES are depicted in Algorithm \ref{alg:ives}.

\begin{algorithm}[t]
	\caption{Iterative Solution (IVES)}
	\label{alg:ives}
	\LinesNumbered
	\KwIn{$\eta_1$, $\eta_2$, $S$, $B$, $\{h_i\}$, $\{I_m\}$}
	Initialize $t=0$ and $\delta_0$ by \eqref{eq:delta_0}\;
	\While{not done}{
		Compute RB allocation stratege $\bm{z}^t$ under $\delta^t$ using Kuhn-Munkres algorithm based on \eqref{eq:find_z}\;
		Compute transmission power $\bm{p}^t$ by \eqref{eq:find_optimal_p}\; 
		Update $\delta^{t+1}=T^\mli{co}(\bm{z}^t,\bm{p}^t)$\;
	}
	\KwRet{$\bm{z}^*$, $\bm{p}^*$}
\end{algorithm}

Using IVES, we can solve (SP2) in an iterative manner. In the following theorem, we provide the convergence guarantee for IVES.
\begin{theorem}
	\label{thm:local_minima}
	If we solve (SP2) by IVES, then $\{g_2(\bm{z}^t,\bm{p}^t)\}$ monotonically increases and converges to a unique point. 
\end{theorem}
\begin{proof}[Sketch of Proof]
The result is derived via proving $g(\bm{z}^t,\bm{p}^t)\le g(\bm{z}^{t+1},\hat{\bm{p}}^{t+1})$ and $g(\bm{z}^{t+1},\hat{\bm{p}}^{t+1})\le g(\bm{z}^{t+1},\bm{p}^{t+1})$. The detailed proof is presented in Appendix \ref{proof:thm_local_minima} of the technical report \cite{technicalreport}.
\end{proof}

Although IVES solves (SP2) iteratively, we observe that it can converge extremely fast (often with only two iterations) in the simulation, achieving a low computation complexity.

Combining the solutions of (SP1) and (SP2), we provide the \emph{\underline{U}ser selection and \underline{R}esource \underline{Al}location (URAL)} in Algorithm \ref{alg:ural} to solve the original problem (P). The URAL can simultaneously optimize the convergence speed, training time, and energy consumption via jointly selecting devices and allocating resources. Further, URAL can be directly integrated into the NUFM paradigm in the device selection phase to facilitate the deployment of FML in wireless networks. 

\begin{algorithm}[t]
	\caption{User Selection and Resource Allocation (\algname{URAL}) Algorithm}
	\label{alg:ural}
	\SetKwFunction{ives}{IVES}
	\LinesNumbered
	\KwIn{$\eta_1$, $\eta_2$, $S$, $B$, $\{h_i\}$, $\{I_m\}$ $\{\iota_i\}$, $\{c_i\}$, $\{D_i\}$}
	Compute $\bm{\nu}^*$ by \eqref{eq:find_nu}\tcp*[r]{Solve (SP1)}
	Compute $\bm{z}^*$ and $\bm{p}^*$ by IVES\tcp*[r]{Solve (SP2)}
	\KwRet{$\bm{\nu}^*$, $\bm{z}^*$, $\bm{p}^*$}
	\
\end{algorithm}

\section{Extension to First-Order Approximations}
\label{sec:extension}
Due to the computation of Hessian in local update \eqref{eq:local_update}, it may cause high computational cost for resource-limited IoT devices. In this section, we address this challenge.

There are two common methods used in the literature to reduce the complexity in computing Hessian \cite{fallah2020personalized}: 
\begin{enumerate}
	\item replacing the stochastic gradient by 
	\begin{align}
		\label{eq:first_approx}
		\tilde{\nabla}F_i(\bfw)\approx\tilde{\nabla}f_i\big(\bfw - \alpha\tilde{\nabla} f_i(\bfw,\mathcal{D}_i),\mathcal{D}'_i\big);
	\end{align}
	
	\item replacing the Hessian-gradient product by 
	\begin{align}
		\label{eq:hessian_estimate}
		&\tilde{\nabla}^2 f_i(\bfw,\mathcal{D}''_i)\tilde{\nabla}f_i\big(\bfw - \alpha\tilde{\nabla} f_i(\bfw,\mathcal{D}_i),\mathcal{D}'_i\big)\nonumber\\
		&\approx \frac{\tilde{\nabla}f_i\big(\bfw+\epsilon \tilde{g}_i,\mathcal{D}''_i\big)-\tilde{\nabla}f_i\big(\bfw-\epsilon \tilde{g}_i,\mathcal{D}''_i\big)}{2\epsilon}
	\end{align}
	where $\tilde{g}_i=\tilde{\nabla}f_i\big(\bfw - \alpha\tilde{\nabla} f_i(\bfw,\mathcal{D}_i),\mathcal{D}'_i\big)$.
\end{enumerate}
By doing so, the computational complexity of a one-step local update can be reduced from $O(d^2)$ to $O(d)$, while not sacrificing too much learning performance. Next, we show that our results in Theorem \ref{thm:gap} hold in the above two cases.
\begin{corollary}
	\label{coro:extension}
	Suppose that Assumptions \ref{assump:smoothness}-\ref{assump:similarity} are satisfied, and $\mathcal{D}_i$, $\mathcal{D}'_i$, and $\mathcal{D}''_i$ are independent batches. If the local update and global aggregation follow \eqref{eq:local_update} and \eqref{eq:global_agregation} respectively, we have for $\tau=1$
	\begin{align}
		\mathbb{E}[F(\bfw^k) - F(\bfw^{k+1})]\ge\;& \beta\mathbb{E}\Bigg[\frac{1}{n_k}\sum_{i\in\mathcal{N}_k}\Bigg(\Big(1-\frac{L_F\beta}{2}\Big)\big\|\tilde{\nabla} F_i(\bfw^k)\big\|^2\nonumber\nonumber\\
		&-\Big(\sqrt{\left(1 + \alpha L\right)^2\gamma_G + \alpha \zeta \gamma_H}+\tilde{\sigma}_{F_i}\Big)\nonumber\\
		&\times\sqrt{ \mathbb{E}\Big[\big\|\tilde{\nabla} F_i(\bfw^k)\big\|^2\Big | \mathcal{N}_k\Big]}\Bigg)\Bigg]
	\end{align}
	where $\tilde{\sigma}_{F_i}$ is defined as follows 
	\begin{align}
		\tilde{\sigma}^2_{F_i}=
		\begin{cases} 		  
			4\sigma_{G}^{2}\left(\frac{1}{D^{\prime}_i}+\frac{(\alpha L)^{2}}{D_i}\right) + 2(\alpha L\zeta)^2&,~\mathrm{if~using}~\eqref{eq:first_approx}\\
			\begin{aligned}[c]
				&6 \sigma_{G}^{2}\left(\frac{\alpha^{2}}{ \epsilon^{2} D^{\prime \prime}_i}+\left(1 + 2(\alpha L)^2\right)\right.\\	&\left.\cdot\left(\frac{1}{D^{\prime}_i}+\frac{(\alpha L)^{2}}{D_i}\right)\right)+2(\alpha\rho\epsilon)^2\zeta^4
			\end{aligned}
			&,~\mathrm{if~ using}~\eqref{eq:hessian_estimate}.
		\end{cases}
	\end{align}
\end{corollary}
\begin{proof}
The detailed proof is presented in Appendix \ref{proof:corollaries} of the technical report \cite{technicalreport}.
\end{proof}
Corollary \ref{coro:extension} indicates that \algname{NUFM} can be directly combined with the first-order approximation techniques to reduce computational cost. Further, similar to Corollary \ref{coro:gap}, Corollary \ref{coro:extension} can be extended to the multi-step cases.

\section{Simulation}
\label{sec:simulation}

This section evaluates the performance of our proposed algorithms by comparing with existing baselines in real-world datasets. We first present the experimental setup, including the datasets, models, parameters, baselines, and environment. Then we provide our results from various aspects.

\subsection{Experimental Setup}

\begin{figure*}[t]
	\centering
	\includegraphics[width=\textwidth]{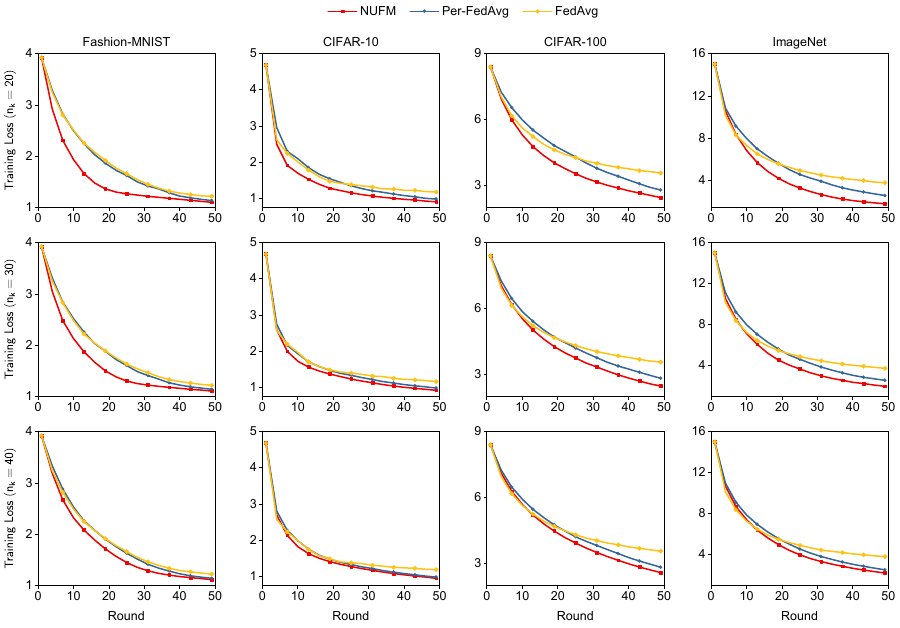}
	\caption{Comparison of convergence rates under different numbers of participated devices. NUFM significantly accelerates the convergence of the existing FML approach, especially with fewer participated devices. In addition, with more participating devices, the advantages of NUFM weaken, leading to smaller gaps between NUFM and the existing algorithms.}
	\label{figure:convergence}
\end{figure*}

\subsubsection{Datasets and Models}

We evaluate our algorithms on four widely-used benchmarks, namely Fashion-MNIST \cite{xiao2017fashion}, CIFAR-10 \cite{krizhevsky2009learning}, CIFAR-100 \cite{krizhevsky2009learning}, and ImageNet \cite{deng2009imagenet}. Specifically, the data is distributed among $n=100$ devices as follows: a) Each device has samples from two random classes; b) the number of samples per class follows a truncated Gaussian distribution $\mathcal{N}(\mu,\sigma^2)$ with $\mu=5$ and $\sigma=5$. We select 50\% devices at random for training with the rest for testing. For each device, we divide the local dataset into a support set and a query set. We consider 1-shot 2-class classification tasks, \ie, the support set contains only 1 labeled example for each class. We set the stepsizes as $\alpha=\beta=0.001$. We use a convolutional neural network (CNN) with max-pooling operations and the Leaky Rectified Linear Unit (Leaky ReLU) activation function, containing three convolutional layers with sizes 32, 64, and 128 respectively, followed by a fully connected layer and a softmax layer. The strides are set as 1 for convolution operation and 2 for the pooling operation. 

\subsubsection{Baselines}

To compare the performance of NUFM, we first consider two existing algorithms, \ie, FedAvg \cite{mcmahan2017communication} and Per-FedAvg \cite{fallah2020personalized}. Further, to validate the effectiveness of URAL in multi-access wireless networks, we use two baselines, called \emph{Greedy} and \emph{Random}. In each round, the \emph{Greedy} strategy determines the CPU-cycle frequency $\nu^\mli{g}_i$ and transmission power $p^\mli{g}_i$ for device $i\in\mathcal{N}$ by greedily minimizing its individual objective, \ie,
\begin{align}
	\label{eq:personal_nu}
	\nu^\mli{g}_i &= \mathop{\arg\min}_{\nu_i}\left\{\eta_1\frac{\iota_i c_i D_i \nu^2_i}{2} + \eta_2\frac{c_i D_i}{\nu_i},~\mathrm{s.t.}~\eqref{eq:constraints_nui}\right\}\\
	\label{eq:personal_p}
	p^\mli{g}_i &= \mathop{\arg\min}_{p_i}\left\{\sum_{m\in\mathcal{M}}\frac{z^\mli{g}_{i,m}(\eta_1 p_i + \eta_2)S}{B\log_2\left(\frac{1+h_i p_i}{I_{m} + BN_0}\right)},~\mathrm{s.t.}~\eqref{eq:constraints_pi}\right\}
\end{align} 
where $\{z^\mli{g}_{i,m}\}_{m\in\mathcal{M}}$ is selected at random (\ie, randomly allocating RBs to selected devices). The \emph{Random} strategy decides the CPU-cycle frequencies, transmission powers, and RB allocation for the selected devices uniformly at random from the feasible regions. 

\subsubsection{Implementation}

We implement the code in TensorFlow Version 1.14 on a server with two Intel$^\circledR$ Xeon$^\circledR$ Golden 5120 CPUs and one Nvidia$^\circledR$ Tesla-V100 32G GPU. The parameters used in the simulation can be found in Table \ref{table:parameters}.

\subsection{Experimental Results}

\subsubsection{Convergence Speed}
\label{subsec:convergence_speed}

To demonstrate the improvement of NUFM on the convergence speed, we compare the algorithms on different benchmarks with the same initial model and learning rate. We vary the number of participated devices $n_k$ from 20 to 40, and set the numbers of local update steps and total communication rounds as $\tau=1$ and $K=50$, respectively. We let $\lambda_1=\lambda_2=1$. As illustrated in Fig. \ref{figure:convergence} and Table \ref{table:accuracy}, NUFM significantly improve the convergence speed and corresponding test accuracy of the existing FML approach on all datasets\footnote{To make the graphs more legible, we draw symbols every two points in Fig. \ref{figure:convergence}.}. Clearly, it validates the effectiveness of our proposed device selection scheme that maximizes the lower bound of one-round global loss reduction. Interestingly, Fig. \ref{figure:convergence} also indicates that NUFM converges more quickly with relatively fewer participated devices. For example, in round 19, the loss achieved by NUFM with $n_k=20$ decreases by more than 9\% and 20\% over those with $n_k=30$ and $n_k=40$ on Fashion-MNIST, respectively. The underlying rationale is that relatively fewer ``good'' devices can provide a larger lower bound on the one-round global loss decrease (note that in \eqref{eq:thm1_1} the lower bound takes the average of the selected devices). More selected devices in each round generally require more communication resources. Thus, the results reveal the potential of NUFM in applications to resource-limited wireless systems.

\begin{table}[ht]
	\centering
	\label{table:accuracy}
	\setlength{\tabcolsep}{4pt}
	\caption{Test accuracy after fifty rounds of training.}
	\renewcommand\arraystretch{1.5}
	\begin{tabular}{cccccc}
		\hline
		Algorithm  & Fashion-MNIST    & CIFAR-10         & CIFAR-100        & ImageNet         &  \\ \hline
		NUFM       & \textbf{68.04\%} & \textbf{58.80\%} & \textbf{23.95\%} & \textbf{34.04\%} &  \\
		Per-FedAvg & 62.75\%          & 58.22\%          & 21.49\%          & 30.98\%          &  \\
		FedAvg     & 61.04\%          & 54.31\%          & 10.13\%          & 12.14\%          &  \\ \hline
	\end{tabular}
\end{table}


\subsubsection{Effect of Local Update Steps}
\label{subsec:selected_users}

To show the effect of local update steps on the convergence rate of NUFM, we present results with varying numbers of local update steps $\tau=1,2,\dots,10$ in each round. For clarity of illustration, we compare the loss under different numbers of local steps in the round 19 on Fashion-MNIST and CIFAR-100. Fig. \ref{figure:vary_tau} shows that fewer local update steps lead to a larger gap between the baselines and NUFM, which verifies the theoretical result that a small number of local steps can slow the convergence of FedAvg and Per-FedAvg \cite[Theorem 4.5]{fallah2020personalized}. It also implies that NUFM can improve the computational efficiency of local devices.

\begin{figure}[ht]
	\centering
	\includegraphics[width=\linewidth]{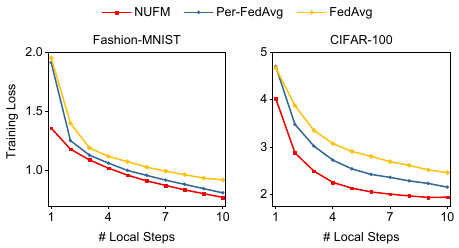}
	\caption{Effect of local update steps on convergence rates. Fewer local steps lead to larger gaps between NUFM and the existing methods.}
	\label{figure:vary_tau}
\end{figure}

%
%

\subsubsection{Performance of URAL in Wireless Networks}
\label{subsec:performance_ural}

We evaluate the performance of URAL by comparing with four baselines, namely NUFM-\emph{Greedy}, NUFM-\emph{Random}, RU-\emph{Greedy}, and RU-\emph{Random}, as detailed below.
\begin{enumerate}[label=\alph*)]
	\item NUFM-\emph{Greedy}: select devices by NUFM, decide CPU-cycle frequencies, RB allocation, and transmission power by \emph{Greedy} strategy;
	\item NUFM-\emph{Random}: select devices by NUFM, decide CPU-cycle frequencies, RB allocation, and transmission power by \emph{Random} strategy;
	\item RU-\emph{Greedy}: select devices uniformly at random, decide CPU-cycle frequencies, RB allocation, and transmission power by \emph{Greedy} strategy;
	\item RU-\emph{Random}: select devices uniformly at random, decide CPU-cycle frequencies, RB allocation, and transmission power by \emph{Random} strategy.
\end{enumerate} 

\begin{table*}[htpb]
	\caption{Parameters in Simulation}
	\label{table:parameters}
	\centering
	\renewcommand\arraystretch{1.25}
	\begin{tabular}{l|l}
		\hline
		\textbf{Parameter} & \textbf{Value} \\ 
		\hline
		Step size ($\alpha$) and meta--learning rate ($\beta$) & 0.001 \\
		\# edge devices ($n$) & 100 \\
		\# participating devices ($n_k$) in Experiments 1 and 2 & An integer varying among $\{20,30,40\}$ \\
		\# local updates ($\tau$) & An integer varying among $\{1,2,\dots,10\}$ \\
		Hyper-parameters ($\lambda_1$ and $\lambda_2$) & 1 \\
		Weight parameters ($\eta_1$ and $\eta_2$) & Real numbers varying among $\{0.5, 1, 1.5, 2.0, 2.5\}$ \\
		\# RBs & An integer varying among $\{1,5,10,20,30,40,50\}$ \\
		CPU cycles of devices($c_i$) & Real numbers following $U(0,0.25)$ \\
		Maximum CPU-cycle frequencies of devices ($\nu^\mli{max}_i$) & Real numbers following $U(0,2)$ \\
		Maximum transmission powers of devices ($p^\mli{max}_i$) & Real numbers following $U(0,1)$ \\
		Inference in RBs ($I_m$) & Real numbers following $U(0,0.8)$ \\
		Model size ($S$), Bandwidth ($B$),  & 1\\
		Noise power spectral density ($N_0$) & 1\\
		effective capacitance coefficients of devices ($\iota_i$) & Real numbers following $U(0,1)$ \\
		Channel gains of devices & 
		\begin{tabular}[c]{@{}l@{}}
			Real numbers following $U(0.1,h^\mli{max})$ \\ 
			where $h^\mli{max}$ varies among $\{0.25, 0.5, 0.75, \dots, 2.0\}$
		\end{tabular} \\
		\hline
	\end{tabular}
\end{table*}

We simulate a wireless system consisting of $M=20$ RBs and let the channel gain $h_i$ of device $i$ follow a uniform distribution $U(h^\mli{min},h^\mli{max})$ with $h^\mli{min}=0.1$ and $h^\mli{max}=1$. We set $S=1$, $B=1$, and $n_0=1$. The interference of RB $m$ is drawn from $I_m\sim U(0,0.8)$. We set $\iota_i\sim U(0,1)$, $c_i\sim U(0,0.25)$, $p^\mli{max}_i\sim U(0,1)$, and $\nu^\mli{max}_i\sim U(0,2)$ for each $i\in\mathcal{N}$ to simulate the device heterogeneity.  In the following experiments, we run the algorithms on FMNIST with local update steps $\tau=1$ and $\eta_1=\eta_2=1$.

As shown in Fig. \ref{figure:ural_comp}, URAL can significantly reduce energy consumption and wall-clock time, as compared with the baselines. However, it is counter-intuitive that the \emph{Greedy} strategy is not always better than \emph{Random}. There are two reasons. On one hand, the energy cost and wall-clock time depend on the selection of weight parameters $\eta_1$ and $\eta_2$. The results in Fig. \ref{figure:ural_comp} imply that, when $\eta_1=\eta_2=1$, \emph{Greedy} pays more attention to the wall-clock time than to energy consumption. Accordingly, \emph{Greedy} achieves much lower average delay than that of \emph{Random}, but sacrificing parts of the energy. On the other hand, the wall-clock time and energy cost require joint control with RB allocation. Although \emph{Greedy} minimizes the individual objectives \eqref{eq:personal_nu}-\eqref{eq:personal_p}, improper RB allocation can cause arbitrary performance degradation. Different from \emph{Greedy} and \emph{Random}--based baselines, since URAL aims to maximize the joint objective (P) via co-optimizing the CPU-cycle frequencies, transmission power, and RB allocation strategy, it can alleviate the above-mentioned issues, and achieve a better delay and energy control. At the same time, Fig. \ref{figure:ural_comp} indicates that URAL converges as fast as NUFM-\emph{Greedy} and NUFM-\emph{Random} (the corresponding lines almost overlap), which select devices greedily to accelerate convergence (as in NUFM). Thus, URAL can have an excellent convergence rate.

\begin{figure}[ht]
	\centering
	\includegraphics[width=\linewidth]{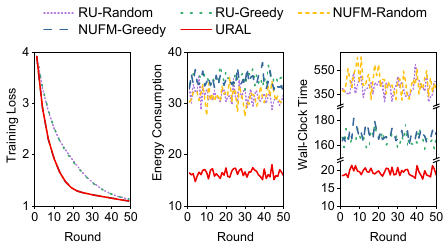}
	\caption{Comparison of convergence, energy cost, and wall-clock training time. URAL can achieve a great convergence speed and short wall-clock time with low energy consumption.}
	\label{figure:ural_comp}
\end{figure}

\subsubsection{Effect of Resource Blocks}
\label{subsec:resource_blocks}

In Fig. \ref{figure:vary_m}, we test the performance of URAL under different numbers of RBs. We vary the number of RBs $M$ from 1 to 50. More RBs enable more devices to be selected in each round, leading to larger energy consumption. As shown in Fig. \ref{figure:vary_m}, URAL keeps stable wall-clock time with the increase of RBs. Meanwhile, URAL can control the power in devices to avoid serious waste of energy. It is counter-intuitive that the convergence speed does not always decrease with the increase of RBs, especially for URAL. The reason is indeed the same as that in Section~\ref{subsec:convergence_speed}. That is, too few selected devices can slow the convergence due to insufficient information provided in each round while a large number of participated devices may weaken the global loss reduction as shown in \eqref{thm:gap}. Therefore, URAL can adapt to the practical systems with constrained wireless resources via achieving fast convergence with only a small set of devices.

\begin{figure}[ht]
	\centering
	\includegraphics[width=\linewidth]{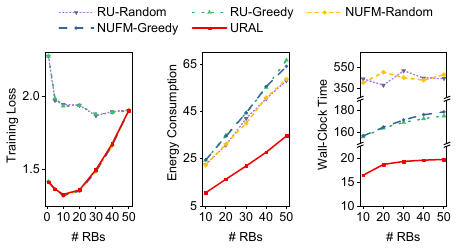}
	\caption{Comparison of convergence, energy cost, and wall-clock time under different numbers of RBs. URAL can well control the energy cost and wall-clock time with more available RBs. Meanwhile, it can achieve fast convergence with only a small number of RBs.}
	\label{figure:vary_m}
\end{figure}

\subsubsection{Effect of Channel Quality}
\label{subsec:channel_quality}

To investigate the effect of channel conditions on performance, we set the number of RBs $M=20$ and vary the maximum channel gain $h^\mli{max}$ from 0.25 to 2, and show the corresponding energy consumption and wall-clock training time in Fig. \ref{figure:vary_h}. The results indicate that the energy consumption and latency decrease as channel quality improves, because devices can use less power to achieve a relatively large transmission rate. 


\begin{figure}[ht]
	\centering
	\includegraphics[width=\linewidth]{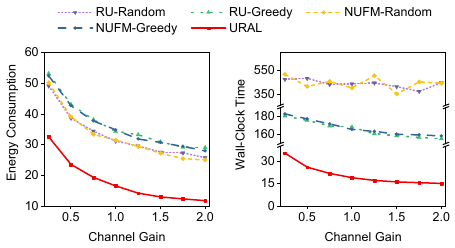}
	\caption{Effect of channel gains on performance. Worse channel conditions would induce larger transmission power and longer wall-clock time.}
	\label{figure:vary_h}
\end{figure}

\subsubsection{Effect of Weight Parameters}
\label{subsec:weight_parameters}

We study how weight parameters $\eta_1$ and $\eta_2$ affect the average energy consumption and wall-clock time of URAL in Fig. \ref{figure:vary_eta}. We first fix $\eta_2=1$ and vary $\eta_1$ from 0.5 to 2.5.  As expected, the total energy consumption decreases with the increase of $\eta_1$, with the opposite trend for wall-clock time. Then we vary $\eta_2$ with $\eta_1=1$. Similarly, a larger $\eta_2$ leads to less latency and more energy cost. It implies that we can control the levels of wall-clock training time and energy consumption by tuning the weight parameters. In particular, even with a large $\eta_1$ or $\eta_2$, the wall-clock time and energy cost can be controlled at low levels. Meanwhile, the convergence rate achieved by URAL is robust to $\eta_1$ and $\eta_2$. Thus, URAL can make full use of the resources (including datasets, bandwidth, and power) and achieve great trade-offs among the convergence rate, latency, and energy consumption.

\begin{figure}[ht]
	\centering
	\includegraphics[width=\linewidth]{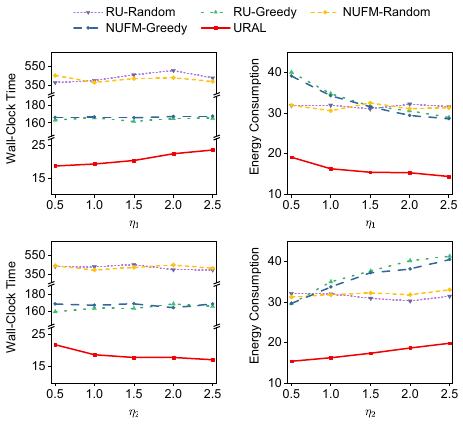}
	\caption{Effect of weight parameters $\eta_1$ and $\eta_2$. A large $\eta_1$ achieves lower energy consumption while leading to longer wall-clock time (the average wall-clock time is 18.63 when $\eta_1=0.5$; it is 21.53 when $\eta_1=2.5$). It is the opposite for $\eta_2$.}
	\label{figure:vary_eta}
\end{figure}

\section{Conclusion}
\label{sec:conclusion}
In this paper, we have proposed an FML algorithm, called NUFM, that maximizes the theoretical lower bound of global loss reduction in each round to accelerate the convergence. Aiming at effectively deploying NUFM in wireless networks, we present a device selection and resource allocation strategy (URAL), which jointly controls the CPU-cycle frequencies and RB allocation to optimize the trade-off between energy consumption and wall-clock training time. Moreover, we integrate the proposed algorithms with two first-order approximation techniques to further reduce the computational complexity in IoT devices. Extensive simulation results demonstrate that the proposed methods outperform the baseline algorithms. 

Future work will investigate the trade-off between the local update and global aggregation in FML to minimize the convergence time and energy cost from a long-term perspective. In addition, how to characterize the convergence properties and communication complexity of the NUFM algorithm requires further research.

\bibliographystyle{IEEEtran}
\bibliography{reference}

\begin{thebibliography}{10}
\providecommand{\url}[1]{#1}
\csname url@samestyle\endcsname
\providecommand{\newblock}{\relax}
\providecommand{\bibinfo}[2]{#2}
\providecommand{\BIBentrySTDinterwordspacing}{\spaceskip=0pt\relax}
\providecommand{\BIBentryALTinterwordstretchfactor}{4}
\providecommand{\BIBentryALTinterwordspacing}{\spaceskip=\fontdimen2\font plus
\BIBentryALTinterwordstretchfactor\fontdimen3\font minus
  \fontdimen4\font\relax}
\providecommand{\BIBforeignlanguage}[2]{{%
\expandafter\ifx\csname l@#1\endcsname\relax
\typeout{** WARNING: IEEEtran.bst: No hyphenation pattern has been}%
\typeout{** loaded for the language `#1'. Using the pattern for}%
\typeout{** the default language instead.}%
\else
\language=\csname l@#1\endcsname
\fi
#2}}
\providecommand{\BIBdecl}{\relax}
\BIBdecl

\bibitem{park2019wireless}
J.~Park, S.~Samarakoon, M.~Bennis, and M.~Debbah, ``Wireless network
  intelligence at the edge,'' \emph{Proc. IEEE}, vol. 107, no.~11, pp.
  2204--2239, 2019.

\bibitem{plastiras2018edge}
G.~Plastiras, M.~Terzi, C.~Kyrkou, and T.~Theocharidcs, ``Edge intelligence:
  Challenges and opportunities of near-sensor machine learning applications,''
  in \emph{Proc. IEEE ASAP}, 2018, pp. 1--7.

\bibitem{zhang2019openei}
X.~Zhang, Y.~Wang, S.~Lu, L.~Liu, W.~Shi \emph{et~al.}, ``Openei: An open
  framework for edge intelligence,'' in \emph{Proc. IEEE ICDCS}, 2019, pp.
  1840--1851.

\bibitem{mcmahan2017communication}
B.~McMahan, E.~Moore, D.~Ramage, S.~Hampson, and B.~A. y~Arcas,
  ``Communication-efficient learning of deep networks from decentralized
  data,'' in \emph{Proc. AISTATS}.\hskip 1em plus 0.5em minus 0.4em\relax PMLR,
  2017, pp. 1273--1282.

\bibitem{finn2017model}
C.~Finn, P.~Abbeel, and S.~Levine, ``Model-agnostic meta-learning for fast
  adaptation of deep networks,'' in \emph{Proc. ICML}, 2017, pp. 1126--1135.

\bibitem{chen2018federated}
F.~Chen, M.~Luo, Z.~Dong, Z.~Li, and X.~He, ``Federated meta-learning with fast
  convergence and efficient communication,'' \emph{ArXiv reprints arXiv:
  1802.07876}, 2019.

\bibitem{jiang2019improving}
Y.~Jiang, J.~Konečný, K.~Rush, and S.~Kannan, ``Improving federated learning
  personalization via model agnostic meta learning,'' \emph{ArXiv reprints
  arXiv: 1909.12488}, 2019.

\bibitem{lin2020collaborative}
S.~Lin, G.~Yang, and J.~Zhang, ``A collaborative learning framework via
  federated meta-learning,'' in \emph{Proc. IEEE ICDCS}, 2020, pp. 289--299.

\bibitem{fallah2020personalized}
A.~Fallah, A.~Mokhtari, and A.~Ozdaglar, ``Personalized federated learning with
  theoretical guarantees: A model-agnostic meta-learning approach,'' in
  \emph{Proc. NIPS}, 2020, pp. 1--12.

\bibitem{kairouz2019advances}
P.~Kairouz \emph{et~al.}, ``Advances and open problems in federated learning,''
  \emph{ArXiv reprints arXiv: 1912.04977}, 2021.

\bibitem{yue2021inexact}
S.~Yue, J.~Ren, J.~Xin, S.~Lin, and J.~Zhang, ``Inexact-admm based federated
  meta-learning for fast and continual edge learning,'' in \emph{Proc. ACM
  MobiHoc}, 2021, p. 91–100.

\bibitem{nishio2019client}
T.~Nishio and R.~Yonetani, ``Client selection for federated learning with
  heterogeneous resources in mobile edge,'' in \emph{Proc. IEEE ICC}, 2019, pp.
  1--7.

\bibitem{nguyen2020fast}
H.~T. Nguyen, V.~Sehwag, S.~Hosseinalipour, C.~G. Brinton, M.~Chiang, and H.~V.
  Poor, ``Fast-convergent federated learning,'' \emph{IEEE J. Sel. Areas
  Commun.}, vol.~39, no.~1, pp. 201--218, 2020.

\bibitem{chen2020joint}
M.~Chen, Z.~Yang, W.~Saad, C.~Yin, H.~V. Poor, and S.~Cui, ``A joint learning
  and communications framework for federated learning over wireless networks,''
  \emph{IEEE Trans. Wireless Commun.}, vol.~20, no.~1, pp. 269--283, 2020.

\bibitem{dinh2020federated}
C.~T. Dinh, N.~H. Tran, M.~N. Nguyen, C.~S. Hong, W.~Bao, A.~Y. Zomaya, and
  V.~Gramoli, ``Federated learning over wireless networks: Convergence analysis
  and resource allocation,'' \emph{IEEE/ACM Trans. Networking}, vol.~29, no.~1,
  pp. 398--409, 2020.

\bibitem{liu2021efficient}
J.~Liu, J.~Ren, Y.~Zhang, X.~Peng, Y.~Zhang, and Y.~Yang, ``Efficient dependent
  task offloading for multiple applications in mec-cloud system,'' \emph{IEEE
  Trans. Mob. Comput.}, 2021.

\bibitem{fadlullah2020hcp}
Z.~M. Fadlullah and N.~Kato, ``Hcp: Heterogeneous computing platform for
  federated learning based collaborative content caching towards 6g networks,''
  \emph{IEEE Trans. Emerging Top. Comput.}, 2020.

\bibitem{wu2020personalized}
Q.~Wu, K.~He, and X.~Chen, ``Personalized federated learning for intelligent
  iot applications: A cloud-edge based framework,'' \emph{IEEE Open J. Comput.
  Soc.}, vol.~1, pp. 35--44, 2020.

\bibitem{yang2019federated}
Q.~Yang, Y.~Liu, T.~Chen, and Y.~Tong, ``Federated machine learning: Concept
  and applications,'' \emph{ACM Trans. Intell. Syst. Technol.}, vol.~10, no.~2,
  pp. 1--19, 2019.

\bibitem{yue2021todg}
S.~Yue, J.~Ren, N.~Qiao, Y.~Zhang, H.~Jiang, Y.~Zhang, and Y.~Yang, ``Todg:
  Distributed task offloading with delay guarantees for edge computing,''
  \emph{IEEE Trans. Parallel Distrib. Syst.}, 2021.

\bibitem{ren2020scheduling}
J.~Ren, Y.~He, D.~Wen, G.~Yu, K.~Huang, and D.~Guo, ``Scheduling for cellular
  federated edge learning with importance and channel awareness,'' \emph{IEEE
  Trans. Wireless Commun.}, vol.~19, no.~11, pp. 7690--7703, 2020.

\bibitem{zhu2019broadband}
G.~Zhu, Y.~Wang, and K.~Huang, ``Broadband analog aggregation for low-latency
  federated edge learning,'' \emph{IEEE Trans. Wireless Commun.}, vol.~19,
  no.~1, pp. 491--506, 2019.

\bibitem{zeng2020energy}
Q.~Zeng, Y.~Du, K.~Huang, and K.~K. Leung, ``Energy-efficient radio resource
  allocation for federated edge learning,'' in \emph{Proc. IEEE ICC Workshops},
  2020, pp. 1--6.

\bibitem{shi2020joint}
W.~Shi, S.~Zhou, Z.~Niu, M.~Jiang, and L.~Geng, ``Joint device scheduling and
  resource allocation for latency constrained wireless federated learning,''
  \emph{IEEE Trans. Wireless Commun.}, vol.~20, no.~1, pp. 453--467, 2020.

\bibitem{chen2020convergence}
M.~Chen, H.~V. Poor, W.~Saad, and S.~Cui, ``Convergence time optimization for
  federated learning over wireless networks,'' \emph{IEEE Trans. Wireless
  Commun.}, vol.~20, no.~4, pp. 2457--2471, 2020.

\bibitem{ren2020accelerating}
J.~Ren, G.~Yu, and G.~Ding, ``Accelerating dnn training in wireless federated
  edge learning systems,'' \emph{IEEE J. Sel. Areas Commun.}, vol.~39, no.~1,
  pp. 219--232, 2020.

\bibitem{yu2019linear}
H.~Yu, R.~Jin, and S.~Yang, ``On the linear speedup analysis of communication
  efficient momentum sgd for distributed non-convex optimization,'' in
  \emph{Proc. ICML}, 2019, pp. 7184--7193.

\bibitem{xu2020client}
J.~Xu and H.~Wang, ``Client selection and bandwidth allocation in wireless
  federated learning networks: A long-term perspective,'' \emph{IEEE Trans.
  Wireless Commun.}, vol.~20, no.~2, pp. 1188--1200, 2020.

\bibitem{wang2019adaptive}
S.~Wang, T.~Tuor, T.~Salonidis, K.~K. Leung, C.~Makaya, T.~He, and K.~Chan,
  ``Adaptive federated learning in resource constrained edge computing
  systems,'' \emph{IEEE J. Sel. Areas in Commun.}, vol.~37, no.~6, pp.
  1205--1221, 2019.

\bibitem{karimireddy2020scaffold}
S.~P. Karimireddy, S.~Kale, M.~Mohri, S.~Reddi, S.~Stich, and A.~T. Suresh,
  ``Scaffold: Stochastic controlled averaging for federated learning,'' in
  \emph{Proc. ICML}, 2020, pp. 5132--5143.

\bibitem{luo2021cost}
B.~Luo, X.~Li, S.~Wang, J.~Huang, and L.~Tassiulas, ``Cost-effective federated
  learning design,'' in \emph{Proc. IEEE INFOCOM}, 2021.

\bibitem{luping2019cmfl}
W.~Luping, W.~Wei, and L.~Bo, ``Cmfl: Mitigating communication overhead for
  federated learning,'' in \emph{Proc. IEEE ICDCS}, 2019, pp. 954--964.

\bibitem{li2019differentially}
J.~Li, M.~Khodak, S.~Caldas, and A.~Talwalkar, ``Differentially private
  meta-learning,'' in \emph{Proc. ICLR}, 2019.

\bibitem{sim2019personalization}
K.~C. Sim \emph{et~al.}, ``Personalization of end-to-end speech recognition on
  mobile devices for named entities,'' in \emph{Proc. IEEE ASRU}, 2019, pp.
  23--30.

\bibitem{shi2020device}
W.~Shi, S.~Zhou, and Z.~Niu, ``Device scheduling with fast convergence for
  wireless federated learning,'' in \emph{Proc. IEEE ICC}, 2020, pp. 1--6.

\bibitem{yang2020energy}
Z.~Yang, M.~Chen, W.~Saad, C.~S. Hong, and M.~Shikh-Bahaei, ``Energy efficient
  federated learning over wireless communication networks,'' \emph{IEEE Trans.
  Wireless Commun.}, vol.~20, no.~3, pp. 1935--1949, 2020.

\bibitem{zhang2020fedpd}
X.~Zhang, M.~Hong, S.~Dhople, W.~Yin, and Y.~Liu, ``Fedpd: A federated learning
  framework with optimal rates and adaptivity to non-iid data,'' \emph{ArXiv
  reprints arXiv: 2005.11418}, 2020.

\bibitem{fallah2020convergence}
A.~Fallah, A.~Mokhtari, and A.~Ozdaglar, ``On the convergence theory of
  gradient-based model-agnostic meta-learning algorithms,'' in \emph{Proc.
  AISTATS}, 2020, pp. 1082--1092.

\bibitem{technicalreport}
S.~Yue, J.~Ren, J.~Xin, D.~Zhang, Y.~Zhang, and W.~Zhuang, ``Efficient
  federated meta-learning over multi-access wireless networks,'' \emph{arXiv
  preprint arXiv:2108.06453}, 2021.

\bibitem{cormen2009introduction}
T.~H. Cormen, C.~E. Leiserson, R.~L. Rivest, and C.~Stein, \emph{Introduction
  to algorithms}.\hskip 1em plus 0.5em minus 0.4em\relax MIT press, 2009.

\bibitem{burd1996processor}
T.~D. Burd and R.~W. Brodersen, ``Processor design for portable systems,''
  \emph{J. VLSI Sig. Proc. Syst. Sig. Image Video Technol.}, vol.~13, no.~2,
  pp. 203--221, 1996.

\bibitem{weisstein2011hungarian}
E.~W. Weisstein, ``Hungarian maximum matching algorithm,''
  \emph{https://mathworld.wolfram.com/}, 2011.

\bibitem{xiao2017fashion}
H.~Xiao, K.~Rasul, and R.~Vollgraf, ``Fashion-mnist: a novel image dataset for
  benchmarking machine learning algorithms,'' \emph{ArXiv reprints arXiv:
  1708.07747}, 2017.

\bibitem{krizhevsky2009learning}
A.~Krizhevsky, ``Learning multiple layers of features from tiny images,''
  \emph{Technical Report TR-2009, University of Toronto}, 2009.

\bibitem{deng2009imagenet}
J.~Deng, W.~Dong, R.~Socher, L.-J. Li, K.~Li, and L.~Fei-Fei, ``Imagenet: A
  large-scale hierarchical image database,'' in \emph{Proc. of IEEE CVPR},
  2009, pp. 248--255.

\end{thebibliography}

\clearpage
\onecolumn
\appendices

\section{Proof of Lemma \ref{lem:smoothness_Fi}}
\label{proof:lem_smooth_Fi}
The proof is standard \cite{yue2021inexact,fallah2020personalized}, and we provide it for completeness. Recalling the definition of $F_i$ as $F_i(\theta) \coloneqq f_i(\theta - \alpha\nabla f_i(\theta))$, we have
	\begin{align*}
		\nabla F_i(\theta)=(I-\alpha \nabla^{2} f_{i}(\theta)) \nabla f_{i}(\theta-\alpha \nabla f_{i}(\theta)).
	\end{align*}
Based on that, for any $\bfw_1,\bfw_2\in\mathbb{R}^d$ we have 
\begin{align}
	\|\nabla F_{i}(\bfw_1)-\nabla F_{i}(\bfw_2)\| =&\;\|(I-\alpha \nabla^{2} f_{i}(\bfw_1)) \nabla f_{i}(\bfw_1-\alpha \nabla f_{i}(\bfw_1))-(I-\alpha \nabla^{2} f_{i}(\bfw_2)) \nabla f_{i}(\bfw_2-\alpha \nabla f_{i}(\bfw_2))\| \nonumber\\
	=&\;\|(I-\alpha \nabla^{2} f_{i}(\bfw_1))(\nabla f_{i}(\bfw_1-\alpha \nabla f_{i}(\bfw_1))-\nabla f_{i}(\bfw_2-\alpha \nabla f_{i}(\bfw_2))) \nonumber\\
	&+((I-\alpha \nabla^{2} f_{i}(\bfw_1))-(I-\alpha \nabla^{2} f_{i}(\bfw_2))) \nabla f_{i}(\bfw_2-\alpha \nabla f_{i}(\bfw_2)) \| \tag*{(adding and subtracting the term $(I-\alpha \nabla^{2} f_{i}(\theta_1))\nabla f_{i}(\theta_2-\alpha \nabla f_{i}(\theta_2)$)}\\
	\leq&\;\underbrace{\|I-\alpha \nabla^{2} f_{i}(\bfw_1)\|\|\nabla f_{i}(\bfw_1-\alpha \nabla f_{i}(\bfw_1))-\nabla f_{i}(\bfw_2-\alpha \nabla f_{i}(\bfw_2))\|}_{(a)} \nonumber\tag*{(from triangle inequality)}\\
	\label{eq:lem1_inequality1}
	&+\alpha\underbrace{\|\nabla^{2} f_{i}(\bfw_1)-\nabla^{2} f_{i}(\bfw_2)\|\|\nabla f_{i}(\bfw_2-\alpha \nabla f_{i}(\bfw_2))\|}_{(b)}.
\end{align}
Then, for $(a)$, the following holds true
\begin{align}
	\label{eq:lem1_a}
	&\|I-\alpha \nabla^{2} f_{i}(\bfw_1)\|\|\nabla f_{i}(\bfw_1-\alpha \nabla f_{i}(\bfw_1))-\nabla f_{i}(\bfw_2-\alpha \nabla f_{i}(\bfw_2))\| \nonumber\\
	\leq\,&(1+\alpha L) L \| \bfw_1-\alpha \nabla f_{i}(\bfw_1))-\bfw_2+\alpha \nabla f_{i}(\bfw_2) \| \nonumber\tag*{(from Assumption \ref{assump:smoothness})}\\
	\leq\,&(1+\alpha L) L(\|\bfw_1-\bfw_2\|+\alpha\|\nabla f_{i}(\bfw_1)-\nabla f_{i}(\bfw_2)\|) \nonumber\tag*{(from triangle inequality)}\\
	\leq\,&(1+\alpha L)^2 L\|\bfw_1-\bfw_2\|.
\end{align}
Regarding $(b)$, it can be shown that
\begin{align}
	\label{eq:lem1_b}
	\|\nabla^{2} f_{i}(\bfw_1)-\nabla^{2} f_{i}(\bfw_2)\|\|\nabla f_{i}(\bfw_2-\alpha \nabla f_{i}(\bfw_2))\|\le\rho\zeta\|\bfw_1-\bfw_2\|.\tag*{(from Assumptions \ref{assump:smoothness} and \ref{assump:lipschitz_hessian})}
\end{align}
Substituting $(a)$ and $(b)$ into \eqref{eq:lem1_inequality1}, we have the result.

\section{Proof of Lemma \ref{lem:variance_Fi}}
\label{proof:lem_variance_Fi}


We rewrite the stochastic gradient $\tilde{\nabla}F_i(\bfw)$ as follows
\begin{align}
	\label{eq:lem_vaniance_1}
	\tilde{\nabla}F_i(\bfw)=\left(I - \alpha\nabla^2 f_i(\bfw) + \delta^*_1\right)\left(\nabla f_i\left(\bfw - \alpha\nabla f_i(\bfw)\right)+\delta^*_2\right)
\end{align}
where $\delta^*_1$ and $\delta^*_2$ are given by
\begin{align}
	\delta^*_1 &= \alpha\left(\nabla^2 f_i(\bfw) - \tilde{\nabla}^2 f_i(\bfw,\mathcal{D}''_i)\right)\\
	\delta^*_2 &= \tilde{\nabla} f_i\big(\bfw - \alpha\tilde{\nabla} f_i(\bfw,\mathcal{D}_i),\mathcal{D}'_i\big) - \nabla f_i\left(\bfw - \alpha\nabla f_i(\bfw)\right).
\end{align}
Note that $\delta^*_1$ and $\delta^*_2$ are independent. Due to Assumption \ref{assump:bounded_variance}, we have 
\begin{align}
	\label{eq:lem_variance_2}
	\mathbb{E}[\delta^*_1] &= 0\\
	\mathbb{E}[\|\delta^*_1\|^2] &\le \frac{(\alpha\sigma_H)^2}{D''_i}.
\end{align}
Next, we proceed to bound the first and second moments of $\delta^*_2$. Regarding the first moment, we have 
\begin{align}
	\left\|\mathbb{E}\left[\delta^*_2\right]\right\|& =  \left\|\mathbb{E}\left[\tilde{\nabla} f_i\big(\bfw - \alpha\tilde{\nabla} f_i(\bfw,\mathcal{D}_i),\mathcal{D}'_i\big) - \nabla f_i\big(\bfw - \alpha\tilde{\nabla} f_i(\bfw,\mathcal{D}_i)\big)+ \nabla f_i\big(\bfw - \alpha\tilde{\nabla} f_i(\bfw,\mathcal{D}_i)\big) - \nabla f_i\left(\bfw - \alpha\nabla f_i(\bfw)\right)\right]\right\|\nonumber\\
	&=  \left\|\mathbb{E}\left[\nabla f_i\big(\bfw - \alpha\tilde{\nabla} f_i(\bfw,\mathcal{D}_i)\big) - \nabla f_i\left(\bfw - \alpha\nabla f_i(\bfw)\right)\right]\right\| \tag*{(from the tower rule and independence between $\mathcal{D}$ and $\mathcal{D}'$)}\nonumber\\
	&\le \mathbb{E}\left[\left\|\nabla f_i\big(\bfw - \alpha\tilde{\nabla} f_i(\bfw,\mathcal{D}_i)\big) - \nabla f_i\left(\bfw - \alpha\nabla f_i(\bfw)\right)\right\|\right]\nonumber\\
	&\le \alpha L\mathbb{E}\left[\left\|\tilde{\nabla} f_i(\bfw,\mathcal{D}_i) - \nabla f_i(\bfw)\right\|\right]\tag*{(from the smoothness of $f_i$)}\\
	&\le\frac{\alpha\sigma_G L}{\sqrt{D_i}}\tag*{(from Assumption \ref{assump:bounded_variance})}.
\end{align}
Regarding the second moment, we have
\begin{align}
	\mathbb{E}\left[\|\delta^*_2\|^2\right]&\le 2\mathbb{E}\left[\left\|\tilde{\nabla} f_i\big(\bfw - \alpha\tilde{\nabla} f_i(\bfw,\mathcal{D}_i),\mathcal{D}'_i\big) - \nabla f_i\big(\bfw - \alpha\tilde{\nabla} f_i(\bfw,\mathcal{D}_i)\big)\right\|^2\right]+2\mathbb{E}\left[\left\|\nabla f_i\big(\bfw - \alpha\tilde{\nabla} f_i(\bfw,\mathcal{D}_i)\big) - \nabla f_i\left(\bfw - \alpha\nabla f_i(\bfw)\right)\right\|^2\right]\nonumber\\
	&\le \frac{2\sigma^2_G}{D'_i} + 2(\alpha L)^2\mathbb{E}\left[\left\|\tilde{\nabla}f_i(\bfw,\mathcal{D}_i) - \nabla f_i(\bfw)\right\|^2\right]\nonumber\tag*{(from the tower rule, smoothness of $f_i$ along with Assumption \ref{assump:bounded_variance})}\\
	&\le 2\sigma^2_G\left(\frac{1}{D'_i}+\frac{(\alpha L)^2}{D_i}\right).
\end{align}
Based on \eqref{eq:lem_vaniance_1}, we have
\begin{align}
	\left\|\mathbb{E}\left[\tilde{\nabla}F_i(\bfw) - \nabla F_i(\bfw)\right]\right\| &=  \left\|\left(I - \alpha\nabla^2 f_i(\bfw)\right)\mathbb{E}[\delta^*_2] + \mathbb{E}[\delta^*_1]\nabla f_i\left(\bfw - \alpha\nabla f_i(\bfw)\right) + \mathbb{E}[\delta^*_1\delta^*_2]\right\|\nonumber\\
	&\le  \left\|I - \alpha\nabla^2 f_i(\bfw)\right\|\left\|\mathbb{E}[\delta^*_2]\right\|\tag*{(from \eqref{eq:lem_variance_2} and submultiplicative property of matrix norm)}\nonumber\\
	&\le  \frac{\alpha\sigma_G L(1+\alpha L)}{\sqrt{D_i}}
\end{align}
which gives us the first result in Lemma \ref{lem:variance_Fi}.

By the submultiplicative property of matrix norm, 
\begin{align}
	\left\|\tilde{\nabla}F_i(\bfw) - \nabla F_i(\bfw)\right\|\le \left\|I - \alpha\nabla^2 f_i(\bfw)\right\|\left\|\delta^*_2\right\| + \left\|\nabla f_i\left(\bfw - \alpha\nabla f_i(\bfw)\right)\right\|\left\|\delta^*_1\right\| +\left\|\delta^*_1\right\|\left\|\delta^*_2\right\|.
\end{align}
Thus, we have
\begin{align}
	\left\|\tilde{\nabla}F_i(\bfw) - \nabla F_i(\bfw)\right\|^2\le 3\left\|I - \alpha\nabla^2 f_i(\bfw)\right\|^2\left\|\delta^*_2\right\|^2 + 3\left\|\nabla f_i\left(\bfw - \alpha\nabla f_i(\bfw)\right)\right\|^2\left\|\delta^*_1\right\|^2 +3\left\|\delta^*_1\right\|^2\left\|\delta^*_2\right\|^2.
\end{align}
Taking expectation on both sizes, we obtain
\begin{align}
	\mathbb{E}\left[\left\|\tilde{\nabla}F_i(\bfw) - \nabla F_i(\bfw)\right\|^2\right]&\le 3\left(1+\alpha L\right)^2\mathbb{E}\left[\left\|\delta^*_2\right\|^2\right] + 3\zeta^2\mathbb{E}\left[\left\|\delta^*_1\right\|^2\right]+3\mathbb{E}\left[\left\|\delta^*_1\right\|^2\right]\mathbb{E}\left[\left\|\delta^*_2\right\|^2\right]\nonumber\tag*{(using the fact that $\|I - \alpha\nabla^2 f_i(\bfw)\|\le1+\alpha L$)}\\
	&\le \frac{6(\alpha\sigma_G\sigma_H)^2}{D''_i}\left(\frac{1}{D'_i}+\frac{(\alpha L)^2}{D_i}\right) + 6\sigma^2_G(1+\alpha L)^2\left(\frac{1}{D'_i}+\frac{(\alpha L)^2}{D_i}\right)+\frac{3(\alpha \zeta\sigma_H)^2}{D''_i}
\end{align}
which completes the proof.

\section{Proof of Lemma \ref{lem:similarity_Fi}}
\label{proof:lem_similarity_Fi}

Recall the definition of $F_i(\bfw)$. We have
\begin{align}
	\label{eq:lem_similarity_3}
	\left\|\nabla F_i(\bfw) - F_j(\bfw)\right\|=\, & \left\|\left(I-\alpha\nabla^2 f_i(\bfw)\right)\nabla f_i\left(\bfw-\alpha\nabla f_i(\bfw)\right) - \left(I-\alpha\nabla^2 f_j(\bfw)\right)\nabla f_j\left(\bfw-\alpha\nabla f_j(\bfw)\right)\right\|\nonumber\\
	\le\, & \left\|\left(I-\alpha\nabla^2 f_i(\bfw)\right)\nabla f_i\left(\bfw-\alpha\nabla f_i(\bfw)\right) - \left(I-\alpha\nabla^2 f_j(\bfw)\right)\nabla f_i\left(\bfw-\alpha\nabla f_i(\bfw)\right)\right\|\nonumber\\
	&+ \left\|\left(I-\alpha\nabla^2 f_j(\bfw)\right)\nabla f_i\left(\bfw-\alpha\nabla f_i(\bfw)\right) - \left(I-\alpha\nabla^2 f_j(\bfw)\right)\nabla f_j\left(\bfw-\alpha\nabla f_j(\bfw)\right)\right\|\nonumber\\
	\le\, & \alpha\underbrace{\left\|\left(\nabla^2 f_i(\bfw)-\nabla^2 f_j(\bfw)\right)\nabla f_i\left(\bfw-\alpha\nabla f_i(\bfw)\right)\right\|}_A\nonumber\\
	& + \underbrace{\left\|\left(I-\alpha\nabla^2 f_j(\bfw)\right)\left(\nabla f_i\left(\bfw-\alpha\nabla f_i(\bfw)\right) - \nabla f_j\left(\bfw-\alpha\nabla f_j(\bfw)\right)\right)\right\|}_B.
\end{align}

Regrading $A$, due to the submultiplicative property of matrix norm, the following holds
\begin{align}
	\label{eq:lem_similarity_1}
	A&\le \left\|\nabla^2 f_i(\bfw)-\nabla^2 f_j(\bfw)\right\|\left\|\nabla f_i\left(\bfw-\alpha\nabla f_i(\bfw)\right)\right\|\nonumber\\
	&\le \zeta\gamma_H
\end{align}
where the last inequality follows from Assumption \ref{assump:similarity}.

Regrading $B$, similarly, we obtain
\begin{align}
	\label{eq:lem_similarity_2}
	B \le\, & \left\|I-\alpha\nabla^2 f_j(\bfw)\right\|\left\|\nabla f_i\left(\bfw-\alpha\nabla f_i(\bfw)\right) - \nabla f_j\left(\bfw-\alpha\nabla f_j(\bfw)\right)\right\|\nonumber\\
	\mathop{\le}^\textbf{(a)}\, & \left(1 + \alpha L\right)\left\|\nabla f_i\left(\bfw-\alpha\nabla f_i(\bfw)\right) - \nabla f_j\left(\bfw-\alpha\nabla f_j(\bfw)\right)\right\|\nonumber\\
	\le\, & \left(1 + \alpha L\right)\Big(\left\|\nabla f_i\left(\bfw-\alpha\nabla f_i(\bfw)\right) - \nabla f_i\left(\bfw-\alpha\nabla f_j(\bfw)\right)\right\|\nonumber\\
	& + \left\|\nabla f_i\left(\bfw-\alpha\nabla f_j(\bfw)\right) - \nabla f_j\left(\bfw-\alpha\nabla f_j(\bfw)\right)\right\|\Big)\nonumber\\
	\mathop{\le}^\text{(b)}\, & \left(1 + \alpha L\right)\left(\alpha L\left\|\nabla f_i(\bfw) - f_j(\bfw)\right\| + \gamma_G\right)\nonumber\\
	\le\, & \left(1 + \alpha L\right)^2\gamma_G
\end{align}
where (a) is derived via triangle inequality, and (b) follows from Assumptions \ref{assump:smoothness} and \ref{assump:similarity}.

Substituting \eqref{eq:lem_similarity_1} and \eqref{eq:lem_similarity_2} in \eqref{eq:lem_similarity_3}, we have
\begin{align}
	\label{eq:gap_tau1}
	\left\|\nabla F_i(\bfw) - F_j(\bfw)\right\|\le \left(1 + \alpha L\right)^2\gamma_G + \alpha \zeta \gamma_H
\end{align}
thereby completing the proof.

\section{Proof of Lemma \ref{lem:sp1_nui}}
\label{proof:lem_sp1_nui}

Since $j$ is the straggler among devices, (SP1) can be equivalently transformed to 
\begin{align}
	\label{prob:trans_sp1}
	\begin{split}
		\mathop{\min}_{\bm{\nu}}&\quad \eta_1\sum_{i\in\mathcal{N}}\frac{\iota_i}{2}c_i D_i \nu^2_i + \eta_2\frac{c_j D_j}{\nu_j}\\
		\mathrm{s.t.}&\quad 0 \le \nu_i \le \nu^\mli{max}_i,~\forall i\in\mathcal{N}\\
		&\quad \frac{c_i D_i \nu_j}{c_j D_j}\le \nu_i,~\forall i\in\mathcal{N}/j.
	\end{split}
\end{align}
Fixing $\nu_j$, for each $i\in\mathcal{N}/j$\footnote{We implicitly assume that $\mathcal{N}/j$ is not empty.}, we can show that the optimal CPU frequency $\nu^*_i$ of the problem \eqref{prob:trans_sp1} can be obtained via solving the following decomposed convex optimization problem
\begin{align}
	\label{prob:subprob_i}
	\begin{split}
		\mathop{\min}_{\nu_i}&\quad \frac{\iota_i}{2}\eta_1 c_i D_i \nu^2_i\\
		\mathrm{s.t.}&\quad 0 \le \nu_i \le \nu^\mli{max}_i\\
		&\quad \frac{c_i D_i \nu_j}{c_j D_j}\le \nu_i.
	\end{split}
\end{align}
If $\nu_j\le \frac{c_j D_j \nu^\mli{max}_i}{c_i D_i}$, the optimal solution of problem \eqref{prob:subprob_i} is given by
\begin{align}
	\label{eq:solution_vi}
	\nu^*_i = \frac{c_i D_i \nu_j}{c_j D_j}.
\end{align}
Otherwise, the problem is infeasible because the constraints are mutually contradictory. Substituting $\nu_i=\nu^*_i$ in problem \eqref{prob:trans_sp1}, we have the following problem with respect to $\nu_j$
\begin{align}
	\label{prob:subprob_j}
	\begin{split}
		\mathop{\min}_{\nu_j}\,&\quad g(\nu_j) = \underbrace{\eta_1\bigg(\sum_{i\in\mathcal{N}/j}\frac{\iota_i (c_i D_i)^3}{2(c_j D_j)^2} + \frac{\iota_j c_j D_j}{2}\bigg)}_{a_1}\nu^2_j + \underbrace{\eta_2c_j D_j}_{a_2}\frac{1}{\nu_j}\\
		\mathrm{s.t.}&\quad 0 \le \nu_j \le \frac{c_j D_j \nu^\mli{max}_i}{c_i D_i},~\forall i\in\mathcal{N}.
	\end{split}
\end{align}
We simplify the expression of $g(\nu_j)$ as $g(\nu_j)=a_1 \nu^2_j + a_2/\nu_j$ where positive constants $a_1$ and $a_2$ are defined in \eqref{prob:subprob_j}. Then, the derivative of $g(\nu_j)$ can be written as
\begin{align}
	\label{eq:deri_g}
	g'(\nu_j) = 2a_1\nu_j - \frac{a_2}{\nu^2_j}.
\end{align}
Based on \eqref{eq:deri_g}, the minimum value of $g(\nu_j)$ is obtained at its stationary point $\bar{\nu}_j\coloneqq{g'}^{-1}(0)=\sqrt[3]{\frac{a_2}{2a_1}}$. Thus, the optimal solution of problem \eqref{prob:subprob_j} is
\begin{align}
	\label{eq:solution_vj}
	\nu^*_j = \min\left\{\sqrt[3]{\frac{a_2}{2a_1}},\min_{i\in\mathcal{N}}\frac{c_j D_j \nu^\mli{max}_j}{c_i D_i}\right\}.
\end{align}
Combining \eqref{eq:solution_vi} and \eqref{eq:solution_vj}, we complete the proof.

\section{Proof of Lemma \ref{lemma:delta}}
\label{proof:lem_delta}

Given $\bm{z}^*$ and $\delta^*$, $p^*_i$ can be obtained via solving the following problem
\begin{align}
	\label{prob:find_pi}
	\begin{split}
		\mathop{\min}_{p_i}&\quad  \sum_{m\in\mathcal{M}}\frac{ z^*_{i,m} p_i}{\log_2\left(1+\frac{h_i p_i}{I_{m} + BN_0}\right)}\\
		\mathrm{s.t.} &~\,\quad0 \le p_i \le p^\mli{max}_i\\
		&\quad\sum_{m\in\mathcal{M}}\frac{z^*_{i,m}S}{B\log_2\left(1+\frac{h_i p_i}{I_{m} + BN_0}\right)}\le\delta^*
	\end{split}
\end{align}
where we eliminate $u_i$, $\eta_1$, $S$, and $B$ in the objective. Clearly, if $\sum_{m\in\mathcal{M}}z^*_{i,m}=0$, then $p^*_i=0$. If the constraints in \eqref{prob:find_pi} are mutually contradictory, \ie
\begin{align}
	p^\mli{max}_i < \frac{(I_{m^*_i}+BN_0)(2^\frac{S}{B\delta^*}-1)}{h_i}
\end{align}
then $z^*_{i,m}$ must be 0, which gives \eqref{eq:Delta_3}.

When there exists $m^*_i$ such that $z^*_{i,m^*_i}=1$, by denoting $\tilde{p}_i \coloneqq \frac{h_i p_i}{I_{m^*_i} + BN_0}$ and rearranging the terms, we can transform the problem \eqref{prob:find_pi} to 
\begin{align}
	\begin{split}
		\mathop{\min}_{\tilde{p}_i}&\quad  g_3(\tilde{p}_i)= \frac{ \tilde{p}_i}{\log_2\left(1+\tilde{p}_i\right)}\\
		\mathrm{s.t.} &\quad 0 \le \tilde{p}_i \le \frac{h_i p^\mli{max}_i}{I_{m^*_i} + BN_0}\\
		&\quad 2^\frac{S}{B\delta^*} - 1 \le \tilde{p}_i.
	\end{split}
\end{align}
We have
\begin{align}
	\label{eq:g3_gradient}
	g'_3(\tilde{p}_i) = \frac{\log_2(1+\tilde{p}_i)-\tilde{p}_i/\left((1+\tilde{p}_i)\ln2\right)}{\left(\log_2(1+\tilde{p}_i)\right)^2}.
\end{align}
Denoting the numerator of \eqref{eq:g3_gradient} as $f_3(\tilde{p}_i)$, we have $f_3(0) = 0$ and 
\begin{align}
	f'_3(\tilde{p}_i) = \frac{1}{\ln2}\left(\frac{1}{(1+\tilde{p}_i)} - \frac{1}{(1+\tilde{p}_i)^2}\right)>0,\quad\text{when}~\tilde{p}_i>0
\end{align}
which implies that $g_3(\tilde{p}_i)$ is monotonically increasing with $\tilde{p}_i>0$. Thus, recalling the definition of $\tilde{p}_i$, due to \eqref{eq:Delta_3}, we obtain
\begin{align}
	p^*_i = \frac{(I_{m^*_i}+BN_0)(2^\frac{S}{B\delta^*}-1)}{h_i}
\end{align}
which completes the proof of \eqref{eq:Delta_2}.

\section{Proof of Lemma \ref{lem:sp2_pi}}
\label{proof:lem_sp2_pi}

If $j$ denotes the straggler among all devices and $\mathcal{N}^*\neq\emptyset$, we can rewrite \eqref{eq:straggler_co} for each $i\in\mathcal{N}^*$ as
\begin{align}
	\label{eq:constraint_strag}
	&\sum_{m\in\mathcal{M}}z^*_{i,m}\frac{S}{B\log_2\left(1+\frac{h_i p^*_i}{I_{m} + BN_0}\right)} \le \sum_{m\in\mathcal{M}}z^*_{j,m}\frac{S}{B\log_2\left(1+\frac{h_j p^*_j}{I_{m} + BN_0}\right)}\nonumber\\
	\Leftrightarrow\quad&\log_2\left(1+\frac{h_j p^*_j}{I_{m^*_j} + BN_0}\right) \le \log_2\left(1+\frac{h_i p^*_i}{I_{m^*_i} + BN_0}\right)\nonumber\\
	\Leftrightarrow\quad&\frac{(I_{m^*_i}+BN_0)h_j}{(I_{m^*_j}+BN_0)h_i}p^*_j\le p^*_i.
\end{align}
Based on \eqref{eq:constraint_strag}, by eliminating constants $B$ and $S$, we transform (SP2) under fixed $\bm{z}^*$ to 
\begin{align}
	\label{prob:transformed_sp2}
	\begin{split}
		\mathop{\min}_{\{p_i\}_{i\in\mathcal{N}^*}}&\quad \sum_{i\in\mathcal{N}^*}\eta_1\frac{ p_i}{\log_2\left(1+\frac{h_i p_i}{I_{m^*_i} + BN_0}\right)}+\eta_2\frac{1}{\log_2\left(1+\frac{h_j p_j}{I_{m^*_j} + BN_0}\right)}\\
		\mathrm{s.t.} ~\;&~\,\quad0 \le p_i \le p^\mli{max}_i,~\forall i\in\mathcal{N}^*\\
		&~\quad\frac{(I_{m^*_i}+BN_0)h_j}{(I_{m^*_j}+BN_0)h_i}p_j\le p_i,~\forall i\in\mathcal{N}^*/j.
	\end{split}
\end{align}
Due to \eqref{eq:constraint_strag}, for each $i\in\mathcal{N}^*/j$, we can obtain the optimal solution $p^*_i$ of \eqref{prob:transformed_sp2} (abusing notations) via solving the decomposed problem as follows
\begin{align}
	\label{prob:decomposed_pi}
	\begin{split}
		\mathop{\min}_{\tilde{p}_i}\,&\quad g_3(\tilde{p}_i)= \frac{ \tilde{p}_i}{\log_2\left(1+\tilde{p}_i\right)}\\
		\mathrm{s.t.} &\quad0 \le \tilde{p}_i \le \frac{h_i p^\mli{max}_i}{I_{m^*_i} + BN_0}\\
		&\quad\frac{h_j}{I_{m^*_j}+BN_0}p_j\le \tilde{p}_i
	\end{split}
\end{align}
where $\tilde{p}_i \coloneqq \frac{h_i p_i}{I_{m^*_i} + BN_0}$.
Note that in \eqref{prob:decomposed_pi} we consider the nontrivial case where $\mathcal{N}/j$ is not empty. 
It is easy to see that 
\begin{align}
	\label{eq:g3_derivative}
	g'_3(\tilde{p}_i) = \frac{\log_2(1+\tilde{p}_i)-\tilde{p}_i/\left((1+\tilde{p}_i)\ln2\right)}{\left(\log_2(1+\tilde{p}_i)\right)^2}.
\end{align}
Denoting the numerator of \eqref{eq:g3_derivative} as $f_3(\tilde{p}_i)$, we have $f(0) = 0$ and 
\begin{align}
	f'_3(\tilde{p}_i) = \frac{1}{\ln2}\left(\frac{1}{(1+\tilde{p}_i)} - \frac{1}{(1+\tilde{p}_i)^2}\right)>0,\quad\text{when}~\tilde{p}_i>0
\end{align}
which implies that $g_3(\tilde{p}_i)$ is monotonically increasing with $\tilde{p}_i>0$. Thus, if $p_j\le\frac{h_i p^\mli{max}_i(I_{m^*_j}+BN_0)}{h_j(I_{m^*_i} + BN_0)}$, recalling the definition of $\tilde{p}_i$, we have
\begin{align}
	p^*_i = \frac{h_j(I_{m^*_i} + BN_0)}{h_i(I_{m^*_j}+BN_0)}p_j,~\forall i\in\mathcal{N}^*/j.
\end{align}
Otherwise, problem \eqref{prob:decomposed_pi} is infeasible.

Similar to \eqref{prob:decomposed_pi}, denote $\tilde{p}_j\coloneqq \frac{h_j p_j}{I_{m^*_j}+BN_0}$ and institute $p_i = p^*_i$ in \eqref{prob:transformed_sp2} (noting that $p_i=\frac{(I_{m^*_i} + BN_0)\tilde{p}_j}{h_i}$), whereby we have the following problem regarding $\tilde{p}_j$
\begin{align}
	\label{prob:g4}
	\begin{split}
		\mathop{\min}_{\tilde{p}_j}&\quad g_4(\tilde{p}_j)= \underbrace{\eta_1\sum_{i\in\mathcal{N}^*}\frac{I_{m^*_i} + BN_0}{h_i}}_{b_1}\cdot\frac{\tilde{p}_j}{\log_2\left(1+\tilde{p}_j\right)}+\frac{\eta_2}{\log_2\left(1+\tilde{p}_j\right)}\\
		\mathrm{s.t.} &\quad0\le\tilde{p}_j\le\frac{h_i p^\mli{max}_i}{I_{m^*_i} + BN_0},~\forall i\in\mathcal{N}^*.
	\end{split}
\end{align}
Denoting positive constant $b_1$ as in \eqref{prob:g4}, we can write $
g_4(\tilde{p}_j)=\frac{b_1\tilde{p}_j}{\log_2(1+\tilde{p}_j)}+\frac{\eta_2}{\log_2(1+\tilde{p}_j)}$ and obtain
\begin{align}
	\label{eq:derivative_g4}
	g'_4(\tilde{p}_j)&=b_1 g'_3(\tilde{p}_j)-\frac{\eta_2}{(1+\tilde{p}_j)\left(\log_2(1+\tilde{p}_j)\right)^2\ln2}\nonumber\\
	&=b_1\frac{\log_2(1+\tilde{p}_j)-\tilde{p}_j/\left((1+\tilde{p}_j)\ln2\right)}{\left(\log_2(1+\tilde{p}_j)\right)^2}-\frac{\eta_2}{(1+\tilde{p}_j)\left(\log_2(1+\tilde{p}_j)\right)^2\ln2}\nonumber\\
	&=\frac{b_1\big((1+\tilde{p}_j)\log_2(1+\tilde{p}_j)\ln2-\tilde{p}_j\big)-\eta_2}{(1+\tilde{p}_j)\left(\log_2(1+\tilde{p}_j)\right)^2\ln2}.
\end{align}
Next, we show that $g_4(\tilde{p}_j)$ has unique minimum point. Denoting the numerator of \eqref{eq:derivative_g4} as $f_4(\tilde{p}_j)$, we have $f_4(0)=-\eta_2\le0$ and 
\begin{align}
	f'_4(\tilde{p}_j) = b_1\log_2(1+\tilde{p}_j)\ln2\ge0,\quad\text{when}~\tilde{p}_j\ge0
\end{align}
which implies $f_4(\tilde{p}_j)$ is monotonically increasing. Besides, it can be shown that $f_4(b_2)\ge0$, where $b_2= 2^{(1+\sqrt{\max\{\frac{\eta_2}{b_1},1\}-1})/\ln2}$. Thus, there must exist a unique point $\tilde{p}^0_j\in(0,b_2]$ such that $g'_4(\tilde{p}^0_j)=f_4(\tilde{p}^0_j)=0$ holds, and it is also the minimum value of $g_4(\tilde{p}_j)$. 

Accordingly, the optimal solution of the problem \eqref{prob:g4} can be expressed as
\begin{align}
	\tilde{p}^*_j=\min\left\{\tilde{p}^0_j,\min_{i\in\mathcal{N^*}}\frac{h_i p^\mli{max}_i}{I_{m^*_i} + BN_0}\right\}.
\end{align}
Therefore, we complete the proof.

\section{Proof of Theorem \ref{thm:gap}}
\label{proof:thm:gap}

From Lemma \ref{lem:smoothness_Fi} (smoothness of $F_i$), for any $\bfw_1,\bfw_2\in\mathbb{R}^d$, we have
\begin{align}
	\label{eq:thm_tau1_1}
	F(\bfw_2)\le F(\bfw_1) + \nabla F(\bfw_1)^\top(\bfw_2-\bfw_1) + \frac{L_F}{2}\|\bfw_2 - \bfw_1\|^2.
\end{align}
Combining \eqref{eq:thm_tau1_1} with \eqref{eq:local_update}, we have
\begin{align}
    \label{eq:lip_bound}
	F(\bfw^{k+1})&\le F(\bfw^k) + \nabla F(\bfw^k)^\top(\bfw^{k+1}-\bfw^k) + \frac{L_F}{2}\|\bfw^{k+1} - \bfw^k\|^2\nonumber\\
	&=F(\bfw^k) - \beta\nabla F(\bfw^k)^\top\left(\frac{1}{n_k}\sum_{i\in\mathcal{N}_k}\tilde{\nabla}F_i(\bfw^k)\right) + \frac{L_F\beta^2}{2}\left\|\frac{1}{n_k}\sum_{i\in\mathcal{N}_k}\tilde{\nabla}F_i(\bfw^k)\right\|^2.
\end{align}
Denoting
\begin{align}
	G^k\coloneqq\beta\nabla F(\bfw^k)^\top\left(\frac{1}{n_k}\sum_{i\in\mathcal{N}_k}\tilde{\nabla}F_i(\bfw^k)\right) - \frac{L_F\beta^2}{2}\left\|\frac{1}{n_k}\sum_{i\in\mathcal{N}_k}\tilde{\nabla}F_i(\bfw^k)\right\|^2
\end{align}
we have $\mathbb{E}[F(\bfw^k) - F(\bfw^{k+1})]\ge \mathbb{E}[G^k]$. Next, we bound $\mathbb{E}[G^k]$ from below. 

First, for any $i\in\mathcal{N}$, we rewrite $\nabla F(\bfw^K)$ as
\begin{align}
	\label{eq:tau1_4}
	\nabla F(\bfw^k) =& \underbrace{\frac{1}{n}\sum_{j\in\mathcal{N}}\nabla F_j(\bfw^k) - \nabla F_i(\bfw^k)}_{A_i}+ \underbrace{ \nabla F_i(\bfw^k) - \tilde{\nabla} F_i(\bfw^k)}_{B_i }+\tilde{\nabla} F_i(\bfw^k)
\end{align}
and bound $\mathbb{E}[\|A_i\|^2]$ by
\begin{align}
	\label{eq:expect_A}
	\mathbb{E}[\|A_i\|^2]&=\mathbb{E}\bigg[\Big\|\nabla F_i(\bfw^k)-\frac{1}{n}\sum_{j\in\mathcal{N}}\nabla F_j(\bfw^k)\Big\|^2\bigg]\nonumber\\
	&= \mathbb{E}\bigg[\Big\|\frac{1}{n}\sum_{j\in\mathcal{N}}\left(\nabla F_j(\bfw^k)-\nabla F_i(\bfw^k)\right)\Big\|^2\bigg]\nonumber\\
	&\mathop{\le}^\text{(a)} \frac{1}{n}\sum_{j\in\mathcal{N}}\mathbb{E}\left[\left\|\nabla F_j(\bfw^k)-\nabla F_i(\bfw^k)\right\|^2\right] \nonumber\\
	&\mathop{\le}^\text{(b)}\left(1 + \alpha L\right)^2\gamma_G + \alpha \zeta \gamma_H
\end{align}
where (b) is derived from Lemma \ref{lem:similarity_Fi}. Inequality (a) follows from the fact that, for any $a_i\in\mathbb{R}^d$ and $b_i\in\mathbb{R}$,
\begin{align}
	\label{eq:inequality_cauthy}
	\left\|\sum_{i\in\mathcal{N}_k}b_i a_i \right\|^2 &= \sum^d_{j=1}\left(\sum_{i\in\mathcal{N}_k}b_i a^{(j)}_i\right)^2\nonumber\\
	&\le \sum^d_{j=1}\left(\sum_{i\in\mathcal{N}_k}\left(a^{(j)}_i\right)^2\right)\left(\sum_{i\in\mathcal{N}_k}b^2_i\right)\nonumber\\
	&= \left(\sum^d_{j=1}\sum_{i\in\mathcal{N}_k}\left(a^{(j)}_i\right)^2\right)\left(\sum_{i\in\mathcal{N}_k}b^2_i\right)\nonumber\\
	&= \left(\sum_{i\in\mathcal{N}_k}\left\|a_i\right\|^2\right)\left(\sum_{i\in\mathcal{N}_k}b^2_i\right)
\end{align}
where $a^{(j)}_i$ denotes the $j$-th coordinate of $a_i$, and the first inequality is derived by using Cauchy-Schwarz inequality.

Regarding $\mathbb{E}[\|B_i \|^2]$, from Lemma \ref{lem:variance_Fi}, 
\begin{align}
	\label{eq:expect_B}
	\mathbb{E}[\|B_i \|^2]\le\sigma^2_{F_i}.
\end{align}
Substituting \eqref{eq:tau1_4} in $\mathbb{E}[G^k]$, we obtain
\begin{align}
	\label{eq:singlestep_1}
	\mathbb{E}[G^k] &=  \mathbb{E}\left[\beta\nabla F(\bfw^k)^\top\left(\frac{1}{n_k}\sum_{i\in\mathcal{N}_k}\tilde{\nabla}F_i(\bfw^k)\right) - \frac{L_F\beta^2}{2}\left\|\frac{1}{n_k}\sum_{i\in\mathcal{N}_k}\tilde{\nabla}F_i(\bfw^k)\right\|^2\right]\nonumber\\
	&=  \mathbb{E}\left[\frac{\beta}{n_k}\sum_{i\in\mathcal{N}_k}\left(A_i+B_i +\tilde{\nabla} F_i(\bfw^k)\right)^\top\tilde{\nabla} F_i(\bfw^k) - \frac{L_F\beta^2}{2}\left\|\frac{1}{n_k}\sum_{i\in\mathcal{N}_k}\tilde{\nabla}F_i(\bfw^k)\right\|^2\right]\nonumber\\
	&=  \mathbb{E}\left[\frac{\beta}{n_k}\sum_{i\in\mathcal{N}_k}\left(A^\top_i \tilde{\nabla} F_i(\bfw^k)+B^\top_i\tilde{\nabla} F_i(\bfw^k)+\left\|\tilde{\nabla} F_i(\bfw^k)\right\|^2\right)\right] - \frac{L_F\beta^2}{2}\mathbb{E}\left[\left\|\frac{1}{n_k}\sum_{i\in\mathcal{N}_k}\tilde{\nabla}F_i(\bfw^k)\right\|^2\right].
\end{align}
Note that, for any random variables $a,b\in\mathbb{R}^d$ and for $c\neq0$,
\begin{align}
	\label{eq:inequality_ab}
	\mathbb{E}\left[a^\top b\right]\ge- \mathbb{E}\left[2\left\|\left(ca\right)^\top\frac{b}{2c}\right\|\right]\ge -\left(c^2\mathbb{E}\left[\|a\|^2\right] + \frac{\mathbb{E}\left[\|b\|^2\right]}{4c^2}\right).
\end{align}
With $g(x)\coloneqq x\mathbb{E}[\|a\|^2]+\mathbb{E}[\|b\|^2]/(4x)$, we have
\begin{align}
	\label{eq:tightness_bound}
	x^*=\mathop{\arg\min}_x g(x) = \sqrt{\frac{\mathbb{E}\left[\left\|b\right\|^2\right]}{4\mathbb{E}\left[\left\|a\right\|^2\right]}}
\end{align}
which implies that if we set $c^2=x^*$, the lower bound in \eqref{eq:inequality_ab} becomes tight. Thus, substituting this in \eqref{eq:inequality_ab} and rearranging the terms, we have
\begin{align}
	\label{eq:inequality_ab_opt}
	\mathbb{E}\left[a^\top b\right]\ge-\sqrt{\mathbb{E}\left[\left\|a\right\|^2\right]\mathbb{E}\left[\left\|b\right\|^2\right]}.
\end{align}


Based on \eqref{eq:inequality_ab_opt} along with \eqref{eq:inequality_cauthy}, due to the tower rule, we can bound $\mathbb{E}[G^k]$ as follows
\begin{align}
	\mathbb{E}[G^k] =\;&  \mathbb{E}\left[\frac{\beta}{n_k}\sum_{i\in\mathcal{N}_k}\left(A^\top_i \tilde{\nabla} F_i(\bfw^k)+B^\top_i\tilde{\nabla} F_i(\bfw^k)+\left\|\tilde{\nabla} F_i(\bfw^k)\right\|^2\right)\right] - \frac{L_F\beta^2}{2}\mathbb{E}\bigg[\Big\|\frac{1}{n_k}\sum_{i\in\mathcal{N}_k}\tilde{\nabla}F_i(\bfw^k)\Big\|^2\bigg] \tag*{(from \eqref{eq:singlestep_1})}\nonumber \\
	=\;& \mathbb{E}\left[\frac{\beta}{n_k}\sum_{i\in\mathcal{N}_k}\left(\mathbb{E}\left[A^\top_i \tilde{\nabla} F_i(\bfw^k)\Big | \mathcal{N}_k\right]+\mathbb{E}\left[B^\top_i\tilde{\nabla} F_i(\bfw^k)\Big| \mathcal{N}_k\right]+\left\|\tilde{\nabla} F_i(\bfw^k)\right\|^2\right)\right] \nonumber\\
	&- \frac{L_F\beta^2}{2}\mathbb{E}\bigg[\Big\|\frac{1}{n_k}\sum_{i\in\mathcal{N}_k}\tilde{\nabla}F_i(\bfw^k)\Big\|^2\bigg] \tag*{(using the tower rule)}\nonumber \\
	\ge\;&\mathbb{E}\left[\frac{\beta}{n_k}\sum_{i\in\mathcal{N}_k}\left(-\sqrt{\mathbb{E}\left[\left\|A_i\right\|^2\Big | \mathcal{N}_k\right]}\sqrt{ \mathbb{E}\Big[\big\|\tilde{\nabla} F_i(\bfw^k)\big\|^2\Big | \mathcal{N}_k\Big]}-\sqrt{\mathbb{E}\left[\left\|B_i \right\|^2\Big | \mathcal{N}_k\right]}\sqrt{ \mathbb{E}\Big[\big\|\tilde{\nabla} F_i(\bfw^k)\big\|^2\Big | \mathcal{N}_k\Big]}\right)\right]\nonumber\\
	&+ \beta\left(1-\frac{L_F\beta}{2}\right)\mathbb{E}\bigg[\frac{1}{n_k}\sum_{i\in\mathcal{N}_k}\Big\|\tilde{\nabla}F_i(\bfw^k)\Big\|^2\bigg] \tag*{(using \eqref{eq:inequality_ab_opt} and \eqref{eq:inequality_cauthy}, and rearranging terms)}\nonumber \\
	\ge\;& \mathbb{E}\left[\frac{\beta}{n_k}\sum_{i\in\mathcal{N}_k}-\left(\sqrt{\left(1 + \alpha L\right)^2\gamma_G + \alpha \zeta \gamma_H}+\sigma_{F_i}\right)\sqrt{ \mathbb{E}\Big[\big\|\tilde{\nabla} F_i(\bfw^k)\big\|^2\Big | \mathcal{N}_k\Big]}\right]\nonumber\tag*{(from \eqref{eq:expect_A} and \eqref{eq:expect_B})}\\
	&+ \beta\left(1-\frac{L_F\beta}{2}\right)\mathbb{E}\bigg[\frac{1}{n_k}\sum_{i\in\mathcal{N}_k}\Big\|\tilde{\nabla}F_i(\bfw^k)\Big\|^2\bigg] \nonumber \\
	\label{eq:lowerbound_Gk}
	=\;& \beta\mathbb{E}\left[\frac{1}{n_k}\sum_{i\in\mathcal{N}_k}\left(1-\frac{L_F\beta}{2}\right)\big\|\tilde{\nabla} F_i(\bfw^k)\big\|^2-\left(\sqrt{\left(1 + \alpha L\right)^2\gamma_G + \alpha \zeta \gamma_H}+\sigma_{F_i}\right)\sqrt{ \mathbb{E}\Big[\big\|\tilde{\nabla} F_i(\bfw^k)\big\|^2\Big | \mathcal{N}_k\Big]}\right]
\end{align}
thereby completing the proof.

 

\section{Proof of Theorem \ref{thm:local_minima}} 
\label{proof:thm_local_minima}

First, it is easy to see that $g_2(\bm{z},\bm{p})$ is upper bounded by $\sum_{i\in\mathcal{N}}u_i$. Besides, due to \eqref{eq:find_z}, we have
\begin{align}
	\bm{z}^{t+1} = \mathop{\arg\max}_{\bm{z}}\left\{\sum_{i\in\mathcal{N}}\sum_{m\in\mathcal{M}}z_{i,m}\left(u_i-\frac{\eta_1\delta^t(I_m+BN_0)(2^\frac{S}{B\delta^t}-1)}{h_i}\right),~\mathrm{s.t.}~\eqref{eq:constraint_zm}-\eqref{eq:constraint_z01}\right\}.
\end{align}
Based on \eqref{eq:optimal_pi}, the optimal transmission power $\hat{\bm{p}}^{t+1}$ corresponding to $\bm{z}^{t+1}$ is given by
\begin{align}
	\hat{p}^{t+1}_i =
	\begin{cases}
		\frac{(I_{m^*_i}+BN_0)(2^\frac{S}{B\delta^t}-1)}{h_i}&,~\text{if}~\sum_{m\in\mathcal{M}}z^{t+1}_{i,m} = 1\\
		0&,~\text{otherwise}
	\end{cases}
\end{align}
where we slightly abuse the notation and denote the RB block allocated to $i$ as $m^*_i$, \ie, $z^{t+1}_{i,m^*_i}=1$. Thus, we have $g(\bm{z}^t,\bm{p}^t)\le g(\bm{z}^{t+1},\hat{\bm{p}}^{t+1})$. From \eqref{eq:find_optimal_p}, $g(\bm{z}^{t+1},\hat{\bm{p}}^{t+1})\le g(\bm{z}^{t+1},\bm{p}^{t+1})$. Using these two inequalities, we obtain
\begin{align}
	g(\bm{z}^t,\bm{p}^t)\le g(\bm{z}^{t+1},\bm{p}^{t+1})
\end{align}
thereby completing the proof.

\section{Proof of Corollaries \ref{coro:gap} and \ref{coro:extension}}
\label{proof:corollaries}

\subsection{Proof of Corollary \ref{coro:gap}}

Denote auxiliary variable $\bfw^{k,t}\coloneqq\frac{1}{n_k}\sum_{i\in\mathcal{N}_k}\bfw^{k,t}_i$ as a ``virtual'' global model in local step $t$ ($\bfw^{k,t}_i$ is defined in \eqref{eq:local_update}). Similar to \eqref{eq:lip_bound}, the following holds
\begin{align}
    F(\bfw^{k,t+1})&\le F(\bfw^{k,t}) + \nabla F(\bfw^{k,t})^\top(\bfw^{k,t+1}-\bfw^{k,t}) + \frac{L_F}{2}\|\bfw^{k,t+1} - \bfw^{k,t}\|^2\tag*{(from Lemma \ref{lem:smoothness_Fi})}\\
    &=F(\bfw^{k,t}) + \nabla F(\bfw^{k,t})^\top\left(\frac{1}{n_k}\sum_{i\in\mathcal{N}_k}\bfw^{k,t+1}_i -\bfw^{k,t}_i\right) + \frac{L_F}{2}\|\bfw^{k,t+1} - \bfw^{k,t}\|^2\tag*{(from the definition of $\theta^{k,t}$)}\\
	&=F(\bfw^{k,t}) - \beta\nabla F(\bfw^{k,t})^\top\left(\frac{1}{n_k}\sum_{i\in\mathcal{N}_k}\tilde{\nabla}F_i(\bfw^{k,t})\right) + \frac{L_F\beta^2}{2}\left\|\frac{1}{n_k}\sum_{i\in\mathcal{N}_k}\tilde{\nabla}F_i(\bfw^{k,t})\right\|^2.\tag*{(from \eqref{eq:local_update})}
\end{align}
Denote 
\begin{align}
    G^{k,t}\coloneqq \frac{\beta}{n_k}\sum_{i\in\mathcal{N}_k}\nabla F(\bfw^{k,t})^\top\tilde{\nabla}F_i(\bfw^{k,t}) - \frac{L_F\beta^2}{2}\left\|\frac{1}{n_k}\sum_{i\in\mathcal{N}_k}\tilde{\nabla}F_i(\bfw^{k,t})\right\|^2.
\end{align}
We have $\mathbb{E}[F(\bfw^{k,t})-F(\bfw^{k,t+1})]\ge\mathbb{E}[G^{k,t}]$. Rewrite $\nabla F(\bfw^{k,t})$ for each $i\in\mathcal{N}_k$ as follows
\begin{align}
    \label{eq:mul_tau_1}
    \nabla F(\bfw^{k,t}) = \underbrace{\nabla F(\bfw^{k,t}) - \nabla F(\bfw^{k,t}_i)}_{C_i} + \underbrace{\nabla F(\bfw^{k,t}_i) - \nabla F_i(\bfw^{k,t}_i)}_{A_i} + \underbrace{\nabla F_i(\bfw^{k,t}_i) - \tilde{\nabla} F_i(\bfw^{k,t}_i)}_{B_i} + \tilde{\nabla} F_i(\bfw^{k,t}_i).
\end{align}
Note that in \eqref{eq:mul_tau_1}, $A_i$ and $B_i$ are similar to those in \eqref{eq:tau1_4} with $\theta^k$ being replaced by $\theta^{k,t}_i$. Thus, from \eqref{eq:expect_A} and \eqref{eq:expect_B}, $\mathbb{E}[\|A_i\|^2]$ and $\mathbb{E}[\|B_i\|^2]$ are bounded by
\begin{align}
    \label{eq:bound_Ai}
    \mathbb{E}[\|A_i\|^2] &\le \left(1 + \alpha L\right)^2\gamma_G + \alpha \zeta \gamma_H\\
    \label{eq:bound_Bi}
    \mathbb{E}[\|B_i\|^2] &\le \sigma^2_{F_i}.
\end{align}
Regarding $\mathbb{E}[\|C_i\|^2]$, we can write 
\begin{align}
    \mathbb{E}[\|C_i\|^2] &= \mathbb{E}\left[\left\|\nabla F(\bfw^{k,t}) - \nabla F(\bfw^{k,t}_i)\right\|^2\right]\nonumber\\
    &\le L^2_F\mathbb{E}\left[\left\|\bfw^{k,t} - \bfw^{k,t}_i\right\|^2\right].\tag*{(From Lemma \ref{lem:smoothness_Fi})}
\end{align}
To bound $\mathbb{E}[\|C_i\|^2]$, denote 
\begin{align}
    a_t\coloneqq\max_{i\in\mathcal{N}_k}\left\{\mathbb{E}\left[\left\|\bfw^{k,t} - \bfw^{k,t}_i\right\|^2\right]\right\}
\end{align}
with $a_0 = 0$. Then we have
\begin{align}
    a_{t+1} &= \max_{i\in\mathcal{N}_k}\left\{\mathbb{E}\left[\left\|\bfw^{k,t+1} - \bfw^{k,t+1}_i\right\|^2\right]\right\}\nonumber\\
    &= \max_{i\in\mathcal{N}_k}\left\{\mathbb{E}\left[\left\|\bfw^{k,t}_i - \beta\tilde{\nabla}F_i(\bfw^{k,t}_i) - \frac{1}{n_k}\sum_{j\in\mathcal{N}_k}\left(\bfw^{k,t}_j - \beta\tilde{\nabla}F_j(\bfw^{k,t}_j)\right) \right\|^2\right]\right\}\nonumber\\
    &\le \max_{i\in\mathcal{N}_k}\left\{(1+\epsilon)\mathbb{E}\left[\left\|\bfw^{k,t}_i - \frac{1}{n_k}\sum_{j\in\mathcal{N}_k}\bfw^{k,t}_j \right\|^2\right] + \beta^2\left(1+\frac{1}{\epsilon}\right)\mathbb{E}\left[\left\|\tilde{\nabla}F_i(\bfw^{k,t}_i) - \frac{1}{n_k}\sum_{j\in\mathcal{N}_k}\tilde{\nabla}F_j(\bfw^{k,t}_j)\right\|^2\right]\right\}\nonumber\tag*{(from \cite[Equation (68)]{fallah2020personalized})}\\
    &\le (1+\epsilon)\max_{i\in\mathcal{N}_k}\left\{\mathbb{E}\left[\left\|\bfw^{k,t}_i - \frac{1}{n_k}\sum_{j\in\mathcal{N}_k}\bfw^{k,t}_j \right\|^2\right]\right\} + \beta^2\left(1+\frac{1}{\epsilon}\right)\max_{i\in\mathcal{N}_k}\left\{\mathbb{E}\left[\left\|\tilde{\nabla}F_i(\bfw^{k,t}_i) - \frac{1}{n_k}\sum_{j\in\mathcal{N}_k}\tilde{\nabla}F_j(\bfw^{k,t}_j)\right\|^2\right]\right\}\nonumber\\
    \label{eq:sequence_a}
    &= (1+\epsilon)a_t + \beta^2\left(1+\frac{1}{\epsilon}\right)\max_{i\in\mathcal{N}_k}\Bigg\{\underbrace{\mathbb{E}\bigg[\bigg\|\tilde{\nabla}F_i(\bfw^{k,t}_i) - \frac{1}{n_k}\sum_{j\in\mathcal{N}_k}\tilde{\nabla}F_j(\bfw^{k,t}_j)\bigg\|^2\bigg]}_{H_i}\Bigg\}
\end{align}
for any $\epsilon>0$. For $H_i$, we first write
\begin{align}
    \label{eq:bound_Hi}
    H_i \le 2\underbrace{\mathbb{E}\left[\left\|\nabla F_i(\bfw^{k,t}_i) - \frac{1}{n_k}\sum_{j\in\mathcal{N}_k}\nabla F_j(\bfw^{k,t}_j)\right\|^2\right]}_{H_{i,1}}+ 2\underbrace{\mathbb{E}\left[\left\| \frac{1}{n_k}\sum_{j\in\mathcal{N}_k}\left(\nabla F_j(\bfw^{k,t}_j)-\tilde{\nabla}F_j(\bfw^{k,t}_j)\right) + \left(\tilde{\nabla}F_i(\bfw^{k,t}_i)-\nabla F_i(\bfw^{k,t}_i)\right)\right\|^2\right]}_{H_{i,2}}.
\end{align}
Next, we bound $H_{i,1}$ and $H_{i,2}$ separately.
\begin{itemize}
    \item \emph{Upper Bound of $H_{i,1}$:} Based on Lemma \ref{lem:similarity_Fi}, we have
    \begin{align}
        \label{eq:H_i1}
        H_{i,1} =\;& \mathbb{E}\left[\left\|\frac{1}{n_k}\sum_{j\in\mathcal{N}_k}\left(\nabla F_i(\bfw^{k,t}) - \nabla F_j(\bfw^{k,t})\right) + \nabla F_i(\bfw^{k,t}_i) - \nabla F_i(\bfw^{k,t}) + \frac{1}{n_k}\sum_{j\in\mathcal{N}_k}\left( \nabla F_j(\bfw^{k,t}) - \nabla F_j(\bfw^{k,t}_j)\right)\right\|^2\right]\nonumber\\
        \le\;& \frac{2}{n_k}\sum_{j\in\mathcal{N}_k}\mathbb{E}\left[\left\|\nabla F_i(\bfw^{k,t}) - \nabla F_j(\bfw^{k,t}) \right\|^2\right]\tag*{(from \eqref{eq:inequality_cauthy})}\\
        &+ 2\mathbb{E}\left[\left\|\nabla F_i(\bfw^{k,t}_i) - \nabla F_i(\bfw^{k,t}) + \frac{1}{n_k}\sum_{j\in\mathcal{N}_k}\left( \nabla F_j(\bfw^{k,t}) - \nabla F_j(\bfw^{k,t}_j)\right)\right\|^2\right]\nonumber\\
        \le\;& 2\gamma^2_G +  2\mathbb{E}\left[\left\|\nabla F_i(\bfw^{k,t}_i) - \nabla F_i(\bfw^{k,t}) + \frac{1}{n_k}\sum_{j\in\mathcal{N}_k}\left( \nabla F_j(\bfw^{k,t}) - \nabla F_j(\bfw^{k,t}_j)\right)\right\|^2\right] \tag*{(from Lemma \ref{lem:similarity_Fi})}\\
        \le\;& 2\gamma^2_G + 4\mathbb{E}\left[\left\|\nabla F_i(\bfw^{k,t}_i) - \nabla F_i(\bfw^{k,t})\right\|^2 + \frac{1}{n_k}\sum_{j\in\mathcal{N}_k} \left\|\nabla F_j(\bfw^{k,t}) - \nabla F_j(\bfw^{k,t}_j)\right\|^2\right]\tag*{(using Cauchy-Schwarz inequality similar to \eqref{eq:mul_tau_2})}\\
        \le\;& 2\gamma^2_G + 4L^2_F\mathbb{E}\left[\left\|\bfw^{k,t}_i - \bfw^{k,t}\right\|^2 + \frac{1}{n_k}\sum_{j\in\mathcal{N}_k} \left\|\bfw^{k,t} - \bfw^{k,t}_j\right\|^2\right]\nonumber\\
        \le\;& 2\gamma^2_G + 8L^2_F a_t.
    \end{align}
    
    \item \emph{Upper Bound of $H_{i,2}$:} Due to Cauchy-Schwarz inequality, the following holds 
    \begin{align}
        \label{eq:mul_tau_2}
        \left\|\sum^{n_k+1}_{j=1}x_j y_j\right\|^2\le\left(\sum^{n_k+1}_{j=1}\left\|x_j\right\|^2\right)\left(\sum^{n_k+1}_{j=1}\left\|y_j\right\|^2\right)
    \end{align}
    where $x_j=\frac{1}{\sqrt{n_k}}$ and $y_j=\frac{1}{\sqrt{n_k}}(\nabla F_j(\bfw^{k,t}_j)-\tilde{\nabla}F_j(\bfw^{k,t}_j))$ for $1\le j\le n_k$; $x_{n_k + 1} = 1$ and $y_{n_k + 1} = \tilde{\nabla}F_i(\bfw^{k,t}_i)-\nabla F_i(\bfw^{k,t}_i)$. Thus, we have
    \begin{align}
        \label{eq:H_i2}
        H_{i,2}&\le 2 \mathbb{E}\left[\left\|\tilde{\nabla}F_i(\bfw^{k,t}_i)-\nabla F_i(\bfw^{k,t}_i)\right\|^2 + \frac{1}{n_k}\sum_{j\in\mathcal{N}_k}\left\|\nabla F_j(\bfw^{k,t}_j)-\tilde{\nabla}F_j(\bfw^{k,t}_j)\right\|^2\right]\nonumber\\
        &\le 4\sigma^2_F \tag*{(using Lemma \ref{lem:variance_Fi})}
    \end{align}
    where $\sigma_F = \max_{i\in\mathcal{N}}\{\sigma_{F_i}\}$.
\end{itemize}
Based on the above results, substituting \eqref{eq:bound_Hi} into \eqref{eq:sequence_a}, we obtain
\begin{align}
    \label{eq:mul_tau_3}
    a_{t+1} &= (1+\epsilon)a_t + 2\beta^2\left(1+\frac{1}{\epsilon}\right)\max_{i\in\mathcal{N}_k}\left\{H_{i,1}+H_{i,2}\right\}\nonumber\\
    &\le \left(1+\epsilon+16\beta^2L^2_F\left(1+\frac{1}{\epsilon}\right)\right)a_t + 4\beta^2\left(1+\frac{1}{\epsilon}\right)\left(\gamma^2_G+2\sigma^2_F\right).
\end{align}
Note that \eqref{eq:mul_tau_3} is essentially the same as \cite[Equation (78)]{fallah2020personalized}. Therefore, due to $\beta\in[0,1/(10L_F\tau))$ and \cite[Corollay F.2.]{fallah2020personalized}, we have
\begin{align}
    \label{eq:bound_Ci}
    \mathbb{E}[\|C_i\|^2]\le a_t \le 35\beta^2 t \tau(\gamma^2_G + 2\sigma^2_F).
\end{align}

Similar to \eqref{eq:lowerbound_Gk}, substituting \eqref{eq:mul_tau_1} in $\mathbb{E}[G^{k,t}]$, we obtain 
\begin{align}
    \mathbb{E}[G^{k,t}] =\;& \mathbb{E}\left[\frac{\beta}{n_k}\sum_{i\in\mathcal{N}_k}\nabla F(\bfw^{k,t})^\top\tilde{\nabla}F_i(\bfw^{k,t}) - \frac{L_F\beta^2}{2}\left\|\frac{1}{n_k}\sum_{i\in\mathcal{N}_k}\tilde{\nabla}F_i(\bfw^{k,t})\right\|^2\right]\nonumber\\
    =\;& \mathbb{E}\left[\frac{\beta}{n_k}\sum_{i\in\mathcal{N}_k}\left(A_i + B_i + C_i + \tilde{\nabla} F_i(\bfw^{k,t}_i)\right)^\top\tilde{\nabla}F_i(\bfw^{k,t}) - \frac{L_F\beta^2}{2}\left\|\frac{1}{n_k}\sum_{i\in\mathcal{N}_k}\tilde{\nabla}F_i(\bfw^{k,t})\right\|^2\right]\nonumber\\
    =\;& \mathbb{E}\left[\frac{\beta}{n_k}\sum_{i\in\mathcal{N}_k}\left(A^\top_i\tilde{\nabla}F_i(\bfw^{k,t}) + B^\top_i\tilde{\nabla}F_i(\bfw^{k,t}) + C^\top_i\tilde{\nabla}F_i(\bfw^{k,t}) + \left\|\tilde{\nabla} F_i(\bfw^{k,t}_i)\right\|^2\right) - \frac{L_F\beta^2}{2}\left\|\frac{1}{n_k}\sum_{i\in\mathcal{N}_k}\tilde{\nabla}F_i(\bfw^{k,t})\right\|^2\right]\nonumber\\
    =\;& \mathbb{E}\left[\frac{\beta}{n_k}\sum_{i\in\mathcal{N}_k}\left(\mathbb{E}\Big[A^\top_i\tilde{\nabla}F_i(\bfw^{k,t})\Big|\mathcal{N}_k\Big] + \mathbb{E}\Big[B^\top_i\tilde{\nabla}F_i(\bfw^{k,t})\Big|\mathcal{N}_k\Big] + \mathbb{E}\Big[C^\top_i\tilde{\nabla}F_i(\bfw^{k,t})\Big|\mathcal{N}_k\Big] + \left\|\tilde{\nabla} F_i(\bfw^{k,t}_i)\right\|^2\right) \nonumber\right.\\
    &\left.- \frac{L_F\beta^2}{2}\left\|\frac{1}{n_k}\sum_{i\in\mathcal{N}_k}\tilde{\nabla}F_i(\bfw^{k,t})\right\|^2\right] \tag*{(using the tower rule)}\nonumber\\
    \ge\;& \mathbb{E}\left[\frac{\beta}{n_k}\sum_{i\in\mathcal{N}_k}\left(-\sqrt{\mathbb{E}\left[\left\|A_i\right\|^2\Big | \mathcal{N}_k\right]}\sqrt{ \mathbb{E}\Big[\big\|\tilde{\nabla} F_i(\bfw^{k,t})\big\|^2\Big | \mathcal{N}_k\Big]}-\sqrt{\mathbb{E}\left[\left\|B_i \right\|^2\Big | \mathcal{N}_k\right]}\sqrt{ \mathbb{E}\Big[\big\|\tilde{\nabla} F_i(\bfw^{k,t})\big\|^2\Big | \mathcal{N}_k\Big]}\right.\right.\nonumber\\
    & -\left.\left. \sqrt{\mathbb{E}\left[\left\|C_i \right\|^2\Big | \mathcal{N}_k\right]}\sqrt{ \mathbb{E}\Big[\big\|\tilde{\nabla} F_i(\bfw^{k,t})\big\|^2\Big | \mathcal{N}_k\Big]}\right)\right] + \beta\left(1-\frac{L_F\beta}{2}\right)\mathbb{E}\bigg[\frac{1}{n_k}\sum_{i\in\mathcal{N}_k}\Big\|\tilde{\nabla}F_i(\bfw^{k,t})\Big\|^2\bigg] \tag*{(using \eqref{eq:inequality_ab_opt} and \eqref{eq:inequality_cauthy}, and rearranging terms)}\\
    \ge\;& \beta\mathbb{E}\left[\frac{1}{n_k}\sum_{i\in\mathcal{N}_k}\left(\left(1-\frac{L_F\beta}{2}\right)\Big\|\tilde{\nabla}F_i(\bfw^{k,t})\Big\|^2\right.\right.\nonumber\\
    &\left.\left.-\left(\sqrt{\left(1 + \alpha L\right)^2\gamma_G + \alpha \zeta \gamma_H}+\sigma_{F_i}+\beta\sqrt{35 t \tau(\gamma^2_G + 2\sigma^2_F)}\right)\sqrt{ \mathbb{E}\Big[\big\|\tilde{\nabla} F_i(\bfw^{k,t})\big\|^2\Big | \mathcal{N}_k\Big]} \right)\right]. \tag*{(from \eqref{eq:bound_Ai}, \eqref{eq:bound_Bi}, and \eqref{eq:bound_Ci})}
\end{align}
Therefore, the following holds
\begin{align}
    \mathbb{E}\left[F(\bfw^k) - F(\bfw^{k+1})\right] &= \mathbb{E}\left[\sum^{\tau-1}_{t=0}F(\bfw^{k,t})-F(\bfw^{k,t+1})\right]\nonumber\\
    &\ge \sum^{\tau-1}_{t=0}\mathbb{E}\left[G^{k,t}\right]\nonumber\\
    &\ge \frac{\beta}{2}\mathbb{E}\left[\frac{1}{n_k}\sum_{i\in\mathcal{N}_k}\sum^{\tau-1}_{t=0}\left(\Big\|\tilde{\nabla}F_i(\bfw^{k,t})\Big\|^2-2\left(\lambda_1+\frac{\lambda_2}{\sqrt{D_i}}\right)\sqrt{ \mathbb{E}\Big[\big\|\tilde{\nabla} F_i(\bfw^{k,t})\big\|^2\Big | \mathcal{N}_k\Big]} \right)\right]\tag*{(substituting \eqref{eq:sigma_Fi})}
\end{align}
where 
\begin{align}
    \lambda_1 &\ge \sqrt{\left(1 + \alpha L\right)^2\gamma_G + \alpha \zeta \gamma_H} + \beta\tau\sqrt{35 (\gamma^2_G + 2\sigma^2_F)}\\
    \lambda_2 &\ge \sqrt{6\sigma^2_G\left(1+(\alpha L)^2\right)\left((\alpha\sigma_H)^2+(1+\alpha L)^2\right)+3(\alpha\zeta\sigma_H)^2}.
\end{align}
We complete the proof.

\subsection{Proof of Corollary \ref{coro:extension}}

The result of Corollary \ref{coro:extension} can be obtained via combining \cite[Lemma H.1]{fallah2020personalized} and \cite[Lemma H.2]{fallah2020personalized} with the proof of Lemma \ref{thm:gap}.

%

\ifCLASSOPTIONcaptionsoff
  \newpage
\fi

\end{document}